\documentclass[wcp,12pt]{alt2021} 
\usepackage[para,online,flushleft]{threeparttable}
\usepackage{mathtools}
\usepackage{multirow}
\usepackage{booktabs}
\usepackage{enumitem}
\newtheorem{thm}{Theorem}
\newtheorem{assumption}{Assumption}
\newtheorem{corol}{Corollary}
\newtheorem{lem}{Lemma}
\newtheorem{prop}{Proposition}
\newtheorem{defn}{Definition}
\newtheorem{rem}{Remark}

\usepackage{algorithm}
\usepackage{algorithmic}

\title[Distributed Online Non-convex Optimization]{Distributed Online Non-convex Optimization with Composite Regret}
\usepackage{times}

\altauthor{%
 \Name{Zhanhong Jiang} \Email{zhanhong.jiang@jci.com}\\
 \addr Johnson Controls
 \AND
 \Name{Aditya Balu} \Email{baditya@iastate.edu}\\
 \addr Iowa State University
 \AND
 \Name{Xian Yeow Lee} \Email{xylee@iastate.edu}\\
 \addr Iowa State University
 \AND
 \Name{Young M. Lee} \Email{young.m.lee@jci.com}\\
 \addr Johnson Controls
 \AND
 \Name{Chinmay Hegde} \Email{chinmay.h@nyu.edu}\\
 \addr New York University
 \AND
 \Name{Soumik Sarkar} \Email{soumiks@iastate.edu}\\
 \addr Iowa State University
}

\begin{document}

\maketitle

\begin{abstract}%
Regret has been widely adopted as the metric of choice for evaluating the performance of online optimization algorithms for distributed, multi-agent systems. However, data/model variations associated with agents can significantly impact decisions, and requires consensus among agents. Moreover, most existing works have focused on developing approaches for (either strongly or non-strongly) convex losses, and very few results have been obtained in terms of regret bounds in distributed online optimization for general non-convex losses. To address these two issues, we propose a novel composite regret with a new network regret-based metric to evaluate distributed online optimization algorithms. We concretely define static and dynamic forms of the composite regret. By leveraging the dynamic form of our composite regret, we develop a consensus-based online normalized gradient (CONGD) approach for pseudo-convex losses and it provably shows a sublinear behavior relating to a regularity term for the path variation of the optimizer. For general non-convex losses, we first shed light on the regret for the setting of distributed online non-convex learning based on the recent advances such that no deterministic algorithm can achieve the sublinear regret. We then develop the distributed online non-convex optimization with composite regret (DINOCO) without access to the gradients, depending on an offline optimization oracle. DINOCO is shown to achieve sublinear regret; to our knowledge, this is the first regret bound for general distributed online non-convex learning.%
\end{abstract}


\section{Introduction}\label{introduction}
\paragraph{Setting.} Online convex optimization (OCO) is by now well established in the algorithmic learning theory literature~\citep{shalev2012online}. However, the emergence of complex tasks and massive datasets requires the investigation of \emph{distributed} online convex optimization for multi-agent networked systems. In this paper, we fill several theoretical gaps in this area.

\paragraph{Prior work.} The performance of a distributed OCO algorithm is generally measured in terms of either the $\textit{static}$~\citep{zinkevich2003online} (or $\textit{dynamic}$~\citep{besbes2015non}) regret, which compares the accumulated losses suffered by each player with the loss suffered by the best fixed (or dynamic) policy. By leveraging static regret, quite a few previous works have been able to achieve sublinear regret bounds. For instance, \citet{nedic2015decentralized} proposed a form of Nesterov's primal-dual algorithm with dual averaging which achieved a sublinear growth $\mathcal{O}(\sqrt{K})$ when the step size was $1/\sqrt{K}$ and the losses were Lipschitz continuous convex functions with Lipschitz gradients.~\citet{akbari2015distributed} improved the order of regret to $\mathcal{O}(\textnormal{log}K)$ when the loss functions were strongly convex with the online subgradient push-sum algorithm. To cope with the presence of cost uncertainties and the switching topologies, \citet{hosseini2016online} developed the distributed weighted dual averaging method and the distributed online allocation algorithm to achieve the sublinear regret of order $\mathcal{O}(\sqrt{K})$. \citet{lee2017sublinear} developed a distributed online primal-dual algorithm with a balanced communication graph and bounded Lagrangian multipliers, which obtained an analogous sublinear regret bound. 

More recent work by \citet{li2018distributed} extended the problem to unbalanced communication graphs while relaxing the assumption that the Lagrangian multipliers were bounded. They also studied the capability of handling coupled inequality constraints by proposing the distributed online primal-dual push-sum algorithm. To determine an effective strategy to handle the constraints, a distributed online saddle point algorithm was proposed in~\citet{paternain2019distributed} to achieve a sublinear regret in the order of $\mathcal{O}(\sqrt{K})$. 

The definition of static regret correspondingly requires that the best policy remains unchanged during the time interval of interest. However, this restriction may not apply to all practical situations (particularly in the distributed setting), therefore motivating the notion of \emph{dynamic} regret. \citet{shahrampour2017distributed} developed a distributed mirror descent method and established a sublinear regret bound $\mathcal{O}(\sqrt{K})$ that was inversely proportional to the spectral gap of the network. \citet{nazari2019dadam} proposed an adaptive gradient method for distributed online optimization (DADAM) using dynamic regret and showed that the proposed method was able to outperform the centralized variant for certain classes of loss functions. However, they defined a local regret for dealing with the nonconvex case instead of using dynamic regret. Recently, \citet{zhang2019distributed} extended the gradient tracking techniques motivated by the recent development~\citep{pu2018distributed} from the offline setting to the online setting with a gradient path variation~\citep{chiang2012online} and proved the sublinear regret. A more recent work~\citep{bastianello2020distributed} showed that when the losses were strongly convex, using an inexact proximal gradient method enabled to converge R-linearly to a neighborhood of the optimal solution trajectory. With long-term constraints, \citet{yuan2019distributed} investigated the regret and the cumulative constraint violation for strongly convex loss functions and obtained a similar regret bound as in the state-of-the-art centralized variants. Additionally, they considered the bandit feedback case and showed the sublinear growth for the regret. In~\citep{sharma2020distributed}, the authors combined mirror descent and primal-dual approaches to achieve a sublinear dynamic regret and fits. To reduce the overhead of communication, the authors in~\citet{cao2021decentralized} developed a decentralized online subgradient algorithm with only the signs of the states from the neighborhood of an agent and showed a standard sublinear regret $\mathcal{O}(\sqrt{K})$, which still provably holds true in a noisy environment. When the feedback delay was considered~\citet{cao2021decentralizedb}, if the delay was sublinear, the regret could still maintain the standard sublinear bound. To investigate the non-cooperative games in a dynamic environment, Meng et al.~\citet{meng2021decentralized} proposed a distributed online primal-dual dynamic mirror descent algorithm, which was found to achieve a relatively worse sublinear regret $\mathcal{O}(K^{\frac{7}{8}})$, compared to the standard regret. When only with bandit feedback~\citet{meng2022decentralizedb}, they showed the regret of $\mathcal{O}(K^{\frac{13}{14}})$. While all the results only apply to convex losses.

In the case of non-convex objective functions, we are only aware of a limited number of (very recent) works. \citet{gao2018online} studied the pseudo-convex loss using the online normalized gradient descent and its bandit variant to show sublinear regret bounds. \citet{hazan2017efficient} and \citet{aydore2018local} respectively proposed different computationally tractable notions of local regret, motivated by offline non-convex optimization. However, the techniques studied do not guarantee a good sublinear regret for general non-convex losses in the online case. The recent work~\citep{agarwal2018learning} utilized the follow-the-perturbed-leader approach, showing that with a good offline optimization oracle (instead of the gradient oracle), the regret can be upper bounded by $\mathcal{O}(K^{\frac{2}{3}})$ for general non-convex losses. Subsequently,~\citet{suggala2019online} improved this bound to $\mathcal{O}(\sqrt{K})$. Such analysis facilitates understanding the regret growth order in online non-convex learning, but unfortunately do not hold for the distributed setting. In a more recent work~\citet{lu2021online}, the authors leveraged a first-order optimality condition-based regret to achieve the subliear regret $\mathcal{O}(\sqrt{K})$ with a diminishing step size. Though they claimed the proposed online distributed mirror algorithm performed well, the techniques presented in their work do not guarantee the vanishing regret for general distributed non-convex losses.

\paragraph{Our contributions.} To summarize, most previous works have investigated distributed online optimization using either the static or dynamic regret, which depends entirely on the original loss objective. However, in a distributed setting with multiple agents, the consensus objective among them can be moved into the loss as a soft constraint, which is a tradeoff between the original loss and the network loss. Moreover, with the popularity of dynamic regret, considering only the temporal change of the path in a distributed setting, involving multiple agents is not sufficient. Different agents may not benefit from a biased best strategy played by one agent. Additionally, there still exists a theoretical gap in the regret bound for distributed online non-convex learning. Motivated by recent works in the offline setting~\citep{jiang2017collaborative,zeng2018nonconvex,berahas2018balancing}, we propose in this paper a composite regret which takes into account both of original loss and network loss.
Specifically, the following contributions are made in this paper.
\begin{itemize}[leftmargin=*]
    \item We propose a composite regret for distributed online optimization for both static and dynamic forms corresponding to different classes of loss functions. 
    We then leverage this to prove regret bounds with pseudo-convex and non-convex losses.
    \item 
    For pseudo-convex loss functions, we propose consensus-based online normalized gradient descent (CONGD). To analyze the regret bound for general non-convex losses, we first extend the result from a recent work to the distributed setting and conclude that no deterministic algorithm can achieve $o(1)$. We then propose the Distributed Online Non-convex Optimization with Composite Regret (DINOCO) based on the follow-the-perturbed-leader, defining an offline optimization oracle.
    \item 
    With mild assumptions, 
    we prove for pseudo-convex losses the regret can be upper bounded by $\mathcal{O}(\sqrt{K+P_K K})$, which matches the best performance in the centralized (non-distributed) variant. Here, $P_K$ is a parameter defined below in Eq~\ref{eq_1}. For general non-convex losses, DINOCO is provably able to achieve the sublinear regret $\mathcal{O}(\sqrt{K})$. Table~\ref{table:findings} in summarizes the comparison between proposed and existing methods. While we refer interested readers to Appendix~\ref{table_section} for more discussion.
    \item Additionally, we develop an online composite gradient descent (OCGD) for strongly and non-strongly convex losses. For the non-strongly convex case, OCGD achieves the sublinear regret $\mathcal{O}(\sqrt{K})$, with properly set step sizes. With strongly convex losses, $\mathcal{O}(\sqrt{K})$ improves to $\mathcal{O}(\textnormal{log}(K))$, given a diminishing step size.
\end{itemize}
\textbf{Outline:} The rest of the paper is outlined as follows. In Section~\ref{lya_regret}, some preliminaries are introduced and the composite regret is discussed in detail. In Section~\ref{pro_alg}, we present the CONGD algorithm with corresponding main result. In Section~\ref{ftpl_lya}, we shed light on the regret bound in the distributed online non-convex learning and develop the DINOCO algorithm that achieves the sublinear regret. Finally, in Section~\ref{conclusion} we give concluding remarks and future research directions.

\begin{table*}[ht]
\caption{Comparisons between different online optimization methods.}
\begin{center}
\begin{threeparttable}
\begin{tabular}{c c c c c c c}
    \toprule
    \textbf{Method} & setting& $\mathcal{F}$ & $\nabla\mathcal{F}$ & \textbf{step size} & \textbf{regret} & \textbf{regret bound} \\ \midrule
      \multirow{2}{*}{OGD}&centralized&Str&B&Di&S&$\mathcal{O}(\textnormal{log}K)$\\&centralized&C&B&Di/F&S&$\mathcal{O}(\sqrt{K})$\\\hline
      OMD&centralized&C&B&F&S&$\mathcal{O}(\sqrt{K})$\\
      RFTL&centralized&C&B&F&S&$\mathcal{O}(\sqrt{K})$\\
      \hline
      \multirow{2}{*}{FTPL}&centralized&C&B&F&S&$\mathcal{O}(\sqrt{K})$\\&centralized&NC&B&F&S&$\mathcal{O}(\sqrt{K})$\\\hline
      ONGD &centralized &PC &B &F&Dy&$\mathcal{O}(\sqrt{K+P_KK})$ \\
      DPD-MD &distributed &C&B &Di&Dy&$\mathcal{O}(K^{\textnormal{max}\{a, 1-a+b\}}+K^a)$ \\
      DOPD-PS&distributed & C&B&Di&S&$\mathcal{O}(K^{\frac{9}{10}})$ \\
      DOPD-DMD&distributed & C&B&Di&Dy&$\mathcal{O}(K^{\textnormal{max}(1-\kappa, \kappa)}+K^{1-\kappa}P_K)$\\
      DOGT  &distributed   &Str&B \& L&F &Dy&$\mathcal{O}(1+P_K+\mathcal{V}_K)$ \\
      ODCMD&distributed&C&B&F&S&$\mathcal{O}(\sqrt{K})$\\
      ODM&distributed&NC&B&Di&FOC&$\mathcal{O}(\sqrt{K})$\\
      DABMD&distributed&C&B&F&Dy&$\mathcal{O}(\sqrt{KP_K+K})$\\
      \hline
      \multirow{2}{*}{OCGD (ours)}&distributed &Str&B&Di&SC&$\mathcal{O}(\textnormal{log}K)$\\&distributed&C&B&Di/F&SC/DC&$\mathcal{O}(\sqrt{K})/\mathcal{O}(\sqrt{K+P_KK})$\\\hline
      CONGD (ours)& distributed&  PC&B&F&DC&$\mathcal{O}(\sqrt{K+P_KK})$  \\
      DINOCO (ours)&distributed&NC&B&F&SC&$\mathcal{O}(\sqrt{K}N^2)$ \\
      \bottomrule
\end{tabular}
\begin{tablenotes}
\item[1] OGD~\citep{hazan2016introduction}: online gradient descent; \item[2] OMD~\citep{hazan2016introduction}: online mirror descent; \item[3] RFTL~\citep{hazan2016introduction}: regularized follow the leader; \item[4] FTPL~\citep{hazan2016introduction}: follow the perturbed leader; \item[5] ONGD~\citep{gao2018online}: online normalized gradient descent; \item[6] $a,b\in(0,1)$; \item[7] DPD-MD~\citep{sharma2020distributed}: distributed primal-dual mirror descent; \item[8] DOPD-PS~\citep{li2018distributed}: distributed online primal-dual push-sum; \item[9] DOPD-DMD~\citep{yi2020distributed}: distributed online primal-dual dynamic mirror descent; \item[10] Str: strongly convex; \item[11] C: convex; \item[12] DOGT~\citep{zhang2019distributed}: distributed online optimization with gradient tracking; \item [13] PC: pseudo-convex; \item[14] ODCMD~\citep{yuan2020distributed}: online distributed composite mirror descent; \item[15] B: bounded; \item[16] B \& L: bounded and Lipschitz continuous; \item[17] Di: diminishing; \item[18] Di/F: diminishing or fixed; \item[19] F: fixed; \item[20] S: static; \item[21] Dy: dynamic; \item[22] SC: static composite; \item[23] DC: dynamic composite; \item[24] $\kappa\in(0,1)$; \item[25] ODM~\citep{lu2021online}: online distributed mirror algorithm; \item[26] FOC: first-order optimality condition; \item[27] DABMD~\citep{eshraghi2020distributed}: Distributed Any-Batch Mirror Descent; \item[28] $P_K, \mathcal{V}_K$: path lengths of the optimizer and gradient variations.
\end{tablenotes}
\end{threeparttable}
\end{center}
\label{table:findings}
\end{table*}

\section{Preliminaries}\label{lya_regret}
\subsection{Distributed Regret}
Consider a networked system involving $N$ agents, which can be described by a graph $\mathcal{G}=(\mathcal{V},\mathcal{E})$, where $\mathcal{V}=\{1,2,...,N\}$ and $\mathcal{E}\subseteq \mathcal{V}\times\mathcal{V}$. In the rest of analysis, the graph $\mathcal{G}$ is assumed to be static and undirected. In distributed online optimization, at each period $k\in\{1,2,...,K\}$, an online player $i\in\{1,2,...,N\}$ individually chooses a feasible policy $x^i_k$ from its decision set known by itself, $\mathcal{X}^i\subset\mathbb{R}^n$, and subsequently incurs a loss $f^i_k(x^i_k): \mathbb{R}^n\to\mathbb{R}$, where $f^i_k(\cdot)$ is a differentiable loss function. In contrast with offline optimization, a remarkably different feature in online optimization is that the player $i$ must choose a policy for the period $k$ without knowing the specific form of loss function $f^i_k(\cdot)$. We next introduce the concepts of distributed regrets.
Specifically, the static regret is defined as
\[\textnormal{S-regret}_K(\mathcal{A}) =\sum_{k=1}^K\sum_{i=1}^Nf^i_k(x^i_k)-\sum_{k=1}^K\sum_{i=1}^Nf^i_k(x_*)\]
where $\mathcal{A}$ is an algorithm, $N$ is the number of agents, $x_*$ is the best strategy learned collaboratively by all agents and 
\[x_*=\underset{x\in\mathcal{X}\subset\mathbb{R}^{n}}{\textnormal{argmin}}\sum_{k=1}^K\sum_{i=1}^Nf^i_k(x),\] 
where $\mathcal{X}:=\cap_{i=1}^N\mathcal{X}^i$. However, to emphasize each agent's individual best strategy, we define $x^i_{*}$ as the best strategy played by the $i$-th agent such that $x^1_*=x^2_*=...=x^N_*=x_*$, which corresponds to all agents communicating sufficiently well. Without loss of generality, we denote the compact form of all the best strategies by $\mathbf{x}_{*}$, which stacks each best strategy from each agent, i.e., $\mathbf{x}_*=[x^1_*;x^2_*;...;x^N_*]\in\mathbb{R}^{nN}$. Hence, it is immediately obtained that $\mathbf{x}_*=[x_*;x_*;...;x_*]_{nN\times 1}$. For the compact form $\mathbf{x}$, it should be noted that the constraint set is expanded to a higher dimension in the form of $\mathcal{X}^N$. For simplicity we assume that $n=1$ in this paper such that $\mathcal{X}^N\subset\mathbb{R}^N$. 

Alternatively, when using the dynamic regret, one can obtain the following definition
\[\textnormal{D-regret}_K(\mathcal{A}) =\sum_{k=1}^K\sum_{i=1}^Nf^i_k(x^i_k)-\sum_{k=1}^K\sum_{i=1}^Nf^i_k(x_{*,k}),\] 
where $x_{*,k}$ can be obtained by
\[x_{*,k}=\underset{x\in\mathcal{X}\subset\mathbb{R}^{n}}{\textnormal{argmin}}\sum_{i=1}^Nf^i_k(x).\]
Similarly, the compact form of the time-varying best strategy is denoted by $\mathbf{x}_{*,k}=[x^1_{*,k};x^2_{*,k};...;x^N_{*,k}]$. Unlike the static regret, the dynamic regret implies that when there is no restriction on the changes of loss functions, the regret is at most linear in $K$ irrespective of the policies~\citep{besbes2015non}. Additionally, to gain meaningful bounds, the temporal change of the loss function sequence $\{f_k\}^K_1$ is assumed to be bounded. Specifically, the losses are assumed to be taken from the set~\citep{besbes2015non,gao2018online}:
\[\mathcal{M}:=\Bigg\{\{f_1,f_2,...,f_K\}:\sum_{k=1}^{K-1}\textnormal{sup}_{x\in\mathcal{X}}|f_k(x)-f_{k+1}(x)|\leq M_K, M_K> 0\Bigg\}.\]
For OCO, \citet{besbes2015non} proved sublinear regret with the above definition. In our analysis we adopt a slightly different regularity parameter defined in terms of the bounded worst-case variation of the optimal solution $x_{*,k}$ of $f_k(\cdot)$:
\begin{equation}\label{eq_1}\mathcal{P}:=\Bigg\{\{f_1,f_2,...,f_K\}:\textnormal{max}_{x_{*,k}\in\textnormal{argmin}_{x\in\mathcal{X}}f_k(x)}\sum_{k=1}^{K-1}\|x_{*,k}-x_{*,k+1}\|\leq P_K, P_K> 0\Bigg\},\end{equation}
where $\|\cdot\|$ is the $\ell_2$ norm.
\subsection{Composite Regret}
Based on the definitions of both static and dynamic regrets, the composite regret is defined formally as follows.
\begin{defn}\label{lyapunov_regret}
The static composite regret of a distributed online algorithm $\mathcal{A}$ is defined as
\begin{equation}\label{lya_s}
\begin{split}
    \textnormal{SC-regret}_K(\mathcal{A})= &\sum_{k=1}^K\sum_{i=1}^Nf^i_k(x^i_k)+c\sum_{k=1}^K\sum_{i=1}^N\sum_{j\in Nb(i)}\pi_{ij}\|x^i_k-x^j_k\|^2\\&-\Bigg(\sum_{k=1}^K\sum_{i=1}^Nf^i_k(x^i_*)+c\sum_{k=1}^K\sum_{i=1}^N\sum_{j\in Nb(i)}\pi_{ij}\|x^i_*-x^j_*\|^2\Bigg),
\end{split}
\end{equation}
where $\Pi = (\pi_{ij}) \in\mathbb{R}^N\times\mathbb{R}^N$, $Nb(i)=\{j\in\mathcal{V}|(i,j)\in\mathcal{E}\}\cup\{i\}$.
Similarly, the dynamic composite regret of a distributed online algorithm $\mathcal{A}$ is defined as
\begin{equation}\label{lya_d}
\begin{split}
    \textnormal{DC-regret}_K(\mathcal{A})= &\sum_{k=1}^K\sum_{i=1}^Nf^i_k(x^i_k)+c\sum_{k=1}^K\sum_{i=1}^N\sum_{j\in Nb(i)}\pi_{ij}\|x^i_k-x^j_k\|^2\\&-\Bigg(\sum_{k=1}^K\sum_{i=1}^Nf^i_k(x^i_{*,k})+c\sum_{k=1}^K\sum_{i=1}^N\sum_{j\in Nb(i)}\pi_{ij}\|x^i_{*,k}-x^j_{*,k}\|^2\Bigg).
\end{split}
\end{equation}
\end{defn}
Based on Definition~\ref{lyapunov_regret}, one extra term of $$c(\sum_{k=1}^K\sum_{i=1}^N\sum_{j\in Nb(i)}\pi_{ij}\|x^i-x^j\|^2 - \sum_{k=1}^K\sum_{i=1}^N\sum_{j\in Nb(i)}\pi_{ij}\|x^i_*-x^j_*\|^2)$$ for $\textnormal{SC-regret}_K$ or $$c(\sum_{k=1}^K\sum_{i=1}^N\sum_{j\in Nb(i)}\pi_{ij}\|x^i_{k}-x^j_{k}\|^2-\sum_{k=1}^K\sum_{i=1}^N\sum_{j\in Nb(i)}\pi_{ij}\|x^i_{*,k}-x^j_{*,k}\|^2)$$ for $\textnormal{DC-regret}_K$ is added accounting for the network loss caused by the disagreement among agents. 
$c$ is a constant parameter that may vary depending on the different methods adopted. For simplicity, $c$ can be set equal to 1, while in this paper, we will show that $c$ can be expressed using the step size.  
To obtain clean analytical results, we rewrite Eqs.~\ref{lya_s} and~\ref{lya_d} in a more compact form. Defining $$V(\mathbf{x}) = \sum_{i=1}^Nf^i(x^i) + c\mathbf{x}^T(\mathbf{I}_N-\Pi)\mathbf{x}=\mathcal{F}(\mathbf{x})+c\mathbf{x}^T(\mathbf{I}_N-\Pi)\mathbf{x},$$ we have
\begin{equation}
    \textnormal{SC-regret}_K(\mathcal{A})=\sum_{k=1}^KV_k(\mathbf{x}_k)-\sum_{k=1}^KV_k(\mathbf{x}_*)
\end{equation}
Likewise, we can obtain the compact form for the dynamic regret as follows.
\begin{equation}
    \textnormal{DC-regret}_K(\mathcal{A})=\sum_{k=1}^KV_k(\mathbf{x}_k)-\sum_{k=1}^KV_k(\mathbf{x}_{*,k})
\end{equation}
The mixing matrix $\Pi$ is assumed to be doubly stochastic, symmetric, with diagonal elements $\pi_{ii}> 0$ and off-diagonal elements $\pi_{ij}> 0$ if and only if $i$ and $j$ are neighbors in the communication network. Under these conditions, $\Pi$ only has one eigenvalue exactly equal to 1 while the rest of eigenvalues are strictly less than 1. We use $\lambda$ to represent the second-largest eigenvalue where $0< \lambda< 1$. The spectral gap, $1-\lambda$, will play a critical role in relating network topology to regret upper bounds.

The composite regret proposed in this context is an extension of the distributed empirical risk minimization in the offline setting. 
Although the regular static and dynamic regrets have been extended straightforward from centralized to distributed online learning, such regrets do not imply good performance, particularly when the agent-wise loss functions are quite different. This scenario could lead to even divergence for some online algorithms. The composite regret enables either the regular static or dynamic regret to have an extra regularization term in the loss functions for coping with the \textit{agent variations} corresponding to the spatial change, compared to the path variation of only one agent shown in Eq.~\ref{eq_1}, corresponding to the temporal change. Therefore, the distributed online learning with the composite regret can simultaneously keep track of both spatial and temporal (hence spatiotemporal) changes along the optimizer path, shown in the regret bound in the rest of the analysis. 
Another advantage of having the composite regret will be illustrated for the distributed online non-convex learning where the gradient oracle is not used.

\section{Proposed Algorithm for Pseudo-convex Losses}
\label{pro_alg}
We present the proposed algorithm for pseudo-convex losses in this section as well as the associated results. Before that, some necessary assumptions and preliminaries are given. 
\begin{assumption}\label{assump_1}
There exists a constant $G> 0$ such that $\mathcal{F}$ is Lipschitz continuous for all $\mathbf{x}, \mathbf{y}\in\mathcal{X}^N$, i.e.,
\begin{equation}
|\mathcal{F}(\mathbf{y})-\mathcal{F}(\mathbf{x})|\leq G\|\mathbf{y}-\mathbf{x}\|.  
\end{equation}
\end{assumption}
\begin{assumption}\label{assump_2}
The set $\mathcal{X}^N\subset\mathbb{R}^N$ is a compact set and  $\forall \mathbf{x}, \mathbf{y}\in\mathcal{X}^N, \|\mathbf{x}-\mathbf{y}\|_{\infty}\leq D$.
\end{assumption}
Assumption~\ref{assump_2} implies immediately $\|\mathbf{x}-\mathbf{y}\|\leq\sqrt{N}\|\mathbf{x}-\mathbf{y}\|_{\infty}\leq\sqrt{N}D$. 
In~\citet{zhang2019distributed}, the authors imposed one additional assumption that each local loss function is smooth, while in this paper, we relax such an assumption as in~\citet{li2018distributed}. Moreover, in most previous works~\citep{li2018distributed,sharma2020distributed}, they had explicit assumptions for the boundedness of $\mathcal{F}$. Similarly, in this paper, we have that $|\mathcal{F}|< \infty$. Assumption~\ref{assump_2} does not indicate whether the set is convex or not, which allows the analysis throughout the paper to use this assumption. However, when analyzing (strongly) convex and pseudo-convex losses, it defaults to a convex set. At the same time, it is not assumed to be so for non-convex analysis. To simplify the analysis in the next section, we also assume that $\mathcal{X}^N$ contains the origin, while in practice, this may not necessarily be true. This simplification allows us to initialize $\mathbf{x}$ at the origin.
\subsection{CONGD for Pseudo-convex Losses}
In literature, numerous existing works have focused on convex optimization, and non-convex problems in the distributed online setting have rarely been explored and investigated with any established methods. In this context, we propose the CONGD for studying the scenario where the loss functions satisfy the pseudo-convexity and show that a sublinear dynamic regret can be achieved with the properly set step size. The online normalized gradient descent (ONGD) has been adopted in~\citet{gao2018online}, while \citet{nesterov2013introductory} proposed the normalized gradient descent. The CONGD utilizes the first-order information $\nabla f^i_k(x^i_k)$ for each agent to compute the normalized vector $\frac{\nabla f^i_k(x^i_k)}{\|\nabla f^i_k(x^i_k)\|}$ as the search direction. The algorithmic framework is presented as follows.
\begin{algorithm}
\caption{Consensus-based Online Normalized Gradient Descent (CONGD)}
\label{congd}
\begin{algorithmic}[1]
 \STATE \textbf{Input:} convex sets $\mathcal{X}^i$, $K$, $x^i_1=0\in\mathcal{X}^i$, step size $\{\alpha_k\}$, $\Pi$, $i\in\{1,...,N\}$
 \FOR{$k=1\;\textnormal{to}\;K$}
    \FOR{each agent $i$}
  \STATE Play the strategy to form $x^i_k$ and observe the loss $f^i_k(x^i_k)$
  \STATE Calculate the feedback $\nabla f^i_k(x^i_k)$
  \IF{$\|\nabla f^i_k(x^i_k)\|> 0$}
  \STATE Update:
    $y^i_{k+1} = 
    \sum_{i=1}^N\pi_{ij}x^j_k - \alpha_k\frac{\nabla f^i_k(x^i_k)}{\|\nabla f^i_k(x^i_k)\|}$\\
  \STATE Project:
    $x^i_{k+1} = \mathbf{P}_{\mathcal{X}^i}(y^i_{k+1})$
    \ELSE
    \STATE $x^i_{k+1}=x^i_k$
    \ENDIF
    \ENDFOR
    \ENDFOR

\end{algorithmic}
\end{algorithm}
$\mathbf{P}_{\mathcal{X}^i}(y)=\textnormal{arg min}_{x\in\mathcal{X}^i}\|y-x\|$ is the Euclidean projection. We denote by $\|\nabla \mathcal{F}_k(\mathbf{x}_k)\| = \|[\nabla f^1_k(x^1_k),...,f^N_k(x^N_k)]^\top\|$. For each agent, they play their own strategies and then observe the corresponding losses, subsequently followed by calculating the feedback. To update the strategy, if the norm of the feedback is not equal to 0, a consensus step is taken for an auxilibary variable $y^i_k$ with the movement in the direction of negative normalized gradient. Then, a projection is taken for the updated auxiliary variable. Otherwise, the current strategy is taken as the next strategy. Algorithm~\ref{congd} shows the agent-wise update for the implementation, while for analysis, we will use a compact form for the ease of the understanding. However, in Line 7, the compact form of the normalized vector is $\bigg[\frac{\nabla f^1_k(x^1_k)}{\|\nabla f^1_k(x^1_k)\|},...,\frac{\nabla f^N_k(x^N_k)}{\|\nabla f^N_k(x^N_k)\|}\bigg]^\top$, which is not exactly equivalent to $\frac{\nabla\mathcal{F}_k(\mathbf{x}_k)}{\|\nabla\mathcal{F}_k(\mathbf{x}_k)\|}$. Hence, in this context, a slight modification is applied for the analysis, i.e., $\bigg[\frac{\nabla f^1_k(x^1_k)}{\|\nabla\mathcal{F}_k(\mathbf{x}_k)\|},...,\frac{\nabla f^N_k(x^N_k)}{\|\nabla\mathcal{F}_k(\mathbf{x}_k)\|}\bigg]^\top$. Due to Assumption~\ref{assump_1}, such a replacement in the analysis will not affect the ultimate convergence rate, while it may lead to slight different constants in the error bound.  Note that such a modification is \textit{only} for the theoretical analysis as in practice, each agent cannot get access to other agents' local gradients.

\begin{rem}
In CONGD, we employ the normalized gradient vector to eliminate the potential negative effect of the gradient's magnitude. 
Another advantage that the normalized vector can bring is to control the speed of convergence and stability in practice by setting an appropriate step size.
\end{rem}
As discussed above, CONGD can be derived in closed form as follows. In order to apply the composite regret to CONGD, such a form helps characterize the analysis for pseudo-convex losses. 
Recall\[\mathbf{y}_{k+1} = 
    \Pi\mathbf{x}_k - \alpha_k\frac{\nabla \mathcal{F}_k(\mathbf{x}_k)}{\|\nabla \mathcal{F}_k(\mathbf{x}_k)\|},\]
which can be rewritten as
\[\mathbf{y}_{k+1} = 
    \mathbf{x}_k - \alpha_k\Bigg(\frac{\nabla \mathcal{F}_k(\mathbf{x}_k)}{\|\nabla \mathcal{F}_k(\mathbf{x}_k)\|}+\frac{1}{\alpha_k}(\mathbf{I}_N-\Pi)\mathbf{x}_k\Bigg).\]
One can observe that the second term on the right-hand side of the last equation is close to a composite gradient, including the normalized gradient and the consensus term. For keeping consistency inside the composite gradient, we also derive the \textit{normalized consensus} term, which will be found to be useful for the subsequent analysis. Thus, we have \[\mathbf{y}_{k+1} = 
    \mathbf{x}_k - \alpha_k\Bigg(\frac{\nabla \mathcal{F}_k(\mathbf{x}_k)}{\|\nabla \mathcal{F}_k(\mathbf{x}_k)\|}+\frac{\|(\mathbf{I}_N-\Pi)\mathbf{x}_k\|}{\alpha_k}\frac{(\mathbf{I}_N-\Pi)\mathbf{x}_k}{\|(\mathbf{I}_N-\Pi)\mathbf{x}_k\|}\Bigg),\] which is \begin{equation}\label{congd_lya}\mathbf{y}_{k+1} = 
    \mathbf{x}_k - \alpha_k\Bigg(\frac{\nabla \mathcal{F}_k(\mathbf{x}_k)}{\|\nabla \mathcal{F}_k(\mathbf{x}_k)\|}+\hat{\alpha}_k\frac{(\mathbf{I}_N-\Pi)\mathbf{x}_k}{\|(\mathbf{I}_N-\Pi)\mathbf{x}_k\|}\Bigg),\end{equation}with $\hat{\alpha}_k=\frac{\|(\mathbf{I}_N-\Pi)\mathbf{x}_k\|}{\alpha_k}$. When $\alpha_k$ is set constant for all $k$, $\hat{\alpha}_k$ can be bounded above by a constant. More details can be found in Appendix~\ref{app_a}. Such a property will help with the analysis when adopting the pseudo-convexity to the loss functions. 
However, we still have to define the constant $c$ for the CONGD. Motivated by the distributed empirical risk minimization in Eq.~\ref{dis_erm_2} in Appendix~\ref{app_f}, $c=\frac{1}{2\alpha_k}, \forall k=1,2,...,K$ is employed in the CONGD. Thus,
the following theorem is shown to claim a sublinear non-stationary regret bound for CONGD.
\begin{thm}\label{theorem_3}
Let all assumptions hold. CONGD with $\textnormal{DC-regret}_K$ and the step size $\alpha_k=\alpha=\sqrt{\frac{\sqrt{N}D(\sqrt{N}D+3P_K)}{K}}$ guarantees the sublinear regret $\mathcal{O}\Bigg(\Bigg(1+\frac{1}{1-\lambda}\Bigg)\sqrt{K(ND^2+3\sqrt{N}DP_K)}\Bigg)$ for all $K\geq 1$ when the loss functions are pseudo-convex.
\end{thm}
The proof is provided in Appendix~\ref{app_b}. It can be observed that for CONGD, the regret growth is the same as the second part of Theorem~\ref{theorem_2} in Appendix~\ref{app_e} for convex losses, with a slightly different step size. Hence, Theorem~\ref{theorem_3} implies that even if the loss function is pseudo-convex, which is weaker than convex, the regret minimization does not differ significantly. It is noted that we omit some constants in the expression of $\mathcal{O}(\cdot)$. The regret bound also implies the impact of network topology such that when the network is dense, which corresponds a smaller $\lambda$ value, the regret bound approaches to $\mathcal{O}\bigg(2\sqrt{K(ND^2+3\sqrt{N}DP_K)}\bigg)$. However, a sparse network results in a larger regret bound.
We present a table in Appendix~\ref{table_section} to show the comparisons between different online optimization methods and give a detailed analysis for difference. 
\section{Proposed Algorithm for Non-convex Losses}\label{ftpl_lya}
As mentioned above, non-convex online learning remains an active research area in the algorithmic learning community. In the centralized setting, recent works~\citep{agarwal2018learning,suggala2019online} have focused on the FTPL algorithm to leverage the perturbation for stabilizing the learning. However, in the distributed setting, no results have been reported to the best of our knowledge. In this paper, we shed light on the availability of the regret bound on top of the existing works. 

\citet{suggala2019online} addressed the natural question whether there exist counterparts of regularized-follow-the-leader (RFTL) and online mirror descent (OMD) which achieve sublinear regret bound, and confirmed that the answer is negative; evidence for this has also been provided in~\citet{agarwal2018learning}. More specifically, there are no deterministic algorithms in the centralized setting that can achieve sublinear regret bound when the losses are general non-convex. We restate their claim as follows.
\begin{prop}~\citep{suggala2019online}\label{prop_1}
No deterministic algorithm can achieve $o(1)$ regret in the setting of online non-convex learning. 
\end{prop}
We remark that with Proposition~\ref{prop_1}, only randomized algorithms can achieve the sublinear regret bound. However, another natural question in this context is ``\textit{Whether there exist counterparts of distributed RFTL and OMD which achieve sublinear regret bound}''.  If the answer is positive, we could extend our proposed algorithms to distributed online non-convex learning and achieve the sublinear regret as the online gradient descent (OGD) type of algorithms can be derived from the OMD and CONGD is a distributed variant of OGD. Regarding the derivation from OMD to OGD please refer to~\citet{hazan2016introduction} for more details. Unfortunately, the answer here is still negative, and we provide a similar formal result as Proposition~\ref{prop_1} in the following.
\begin{prop}\label{prop_2}
No deterministic algorithm can achieve $o(1)$ regret in the setting of \emph{distributed} online non-convex learning. 
\end{prop}
Proposition~\ref{prop_2} is a natural extension from Proposition~\ref{prop_1}. 
The algorithms studied in~\citet{suggala2019online} and~\citet{agarwal2018learning} both relied on a technique, the \textit{offline optimization oracle}, while the former leveraged an approximate one and the latter required a certain one. Qualitatively, the learner perturbs the cumulative loss and submits the perturbed loss to the offline optimization oracle, which then produces a decision. In both works, a linear perturbed term whose random perturbation is sampled from an exponential distribution was added to the cumulative loss. One advantage of using perturbation is that the computational complexity wouldn't increase significantly. 

Motivated by the existing works, in the sequel, we leverage the FTPL with the composite regret to analyze the regret bound for the distributed online non-convex learning.
One simplification that we adopt from~\citet{hazan2016introduction} is that the adversary is \textit{oblivious} such that all losses are adversarially chosen ahead of time and do not depend on the actual decisions of the online learner. 
\subsection{Proposed Algorithm: DINOCO}
Some preliminaries are presented before presenting the algorithmic framework of DINOCO for the setting of distributed non-convex online learning. The perturbation is attained by sampling a random vector from an exponential distribution. Therefore, we denote by $\sigma$ an $N$ dimension random vector from the distribution $\textnormal{exp}(\eta)$, where $\eta> 0$ is the rate parameter for an exponential distribution. Intuitively, $\eta$ plays a similar role as the step size, $\alpha$, and it will be defined as a function of $K$, which significantly affects the growth order of the regret. Also, in Assumption~\ref{assump_1}, we use $l_2$ norm to define the Lipschitz continuity for $\mathcal{F}$. While we adjust the assumption slightly to replace the $l_2$ norm with the $l_1$ norm. Such an adjustment helps characterize the subsequent analysis for the regret bound as the random perturbation vector is sampled coordinate-wise. We will see that in the regret bound, the dimension of such a random perturbation vector also has an impact. By re-defining a constant $\hat{G}> 0$, the conclusion from Assumption~\ref{assump_1} becomes \[|\mathcal{F}(\mathbf{y})-\mathcal{F}(\mathbf{x})|\leq\hat{G}\|\mathbf{y}-\mathbf{x}\|_1.\] We also correspondingly modify the norm of network loss in Eqs.~\ref{lya_s} and~\ref{lya_d}, changing the $l_2$ norm to $l_1$ norm. Such a change does not hurt the concept of composite regret due to the property of the equivalence of norms. Also, this change will not affect the regret bound as the difference is only related to some constants that would be omitted when presenting the regret bound.
To obtain the DINOCO, we adapt an offline optimization oracle with an approximate form~\citep{suggala2019online} and define it formally as follows.
\begin{defn}
For each agent $i$, a ($\rho_i,\beta_i$)-approximate optimization oracle is an oracle that takes as input a function $V^i:\mathcal{X}^i\to\mathbb{R}$ and a random perturbation $\sigma^i$ sampled from an exponential distribution such that\begin{equation}\label{rho_beta}
V^i(x_*)-\sigma^i x_*\leq, \underset{x\in\mathcal{X}^i}{\textnormal{inf}}V^i(x)-\sigma^i x+(\rho+\beta|\sigma^i|),\end{equation}
where $x_*\in\mathcal{X}^i$, $\rho_i, \beta_i > 0$.
\end{defn}

For convenience, we denote by $\mathcal{T}_{\rho_i,\beta_i}(V^i,\sigma^i)$ the ($\rho,\beta$)-approximate optimization oracle for agent $i$.
Given access to $\mathcal{T}_{\rho_i,\beta_i}(V^i,\sigma^i)$, the DINOCO algorithm utilizes the following core prediction rule
\begin{equation}\label{ftpl2}
    x^i_{k+1}=\mathcal{T}_{\rho_i,\beta_i}\Bigg(\sum_{l=1}^kf^i_l+r^i_l,\sigma^i_{k+1}\Bigg)
\end{equation}
where $\sigma^i_{k+1}$ is a random perturbation at the time step $k+1$ sampled from $\textnormal{exp}(\eta)$. $\eta$ will be specified when determining the regret bound. We now formally present the DINOCO in Algorithm~\ref{ftpl2_alg}.

\begin{algorithm}
\caption{DINOCO}
\label{ftpl2_alg}
\begin{algorithmic}[1]
 \STATE \textbf{Input:} rate parameter of exponential distribution $\eta$, ($\rho_i,\beta_i$)-approximate optimization oracle $\mathcal{T}_{\rho_i,\beta_i}$, $\Pi$\\
 \FOR{$k=1\;\textnormal{to}\;K$}
 \FOR{each agent $i$}
  \STATE Agent $i$ plays the strategy to form $x^i_k$ and observe the losses $f^i_k(x^i_k)$ and $r^i_k(x^i_k)=\frac{1}{2\eta}\sum_{j\in Nb(i)}\pi_{ij}\|x^i_k-x^j_k\|^2$
  \STATE Generate random perturbation $\sigma^i_{k+1}\overset{i.i.d}{\sim}\textnormal{exp}(\eta)$
  \STATE Update:
    $x^i_{k+1} = \mathcal{T}_{\rho_i,\beta_i}\Bigg(\sum_{l=1}^k(f^i_l+r^i_l),\sigma^i_{k+1}\Bigg)$
\ENDFOR
 \ENDFOR
\end{algorithmic}
\end{algorithm}
$f^i_l+r^i_l$ is replaced by $V^i_l$ throughout the analysis when presenting the update law. It can be observed from Algorithm~\ref{ftpl2_alg} that each $\sigma^i_{k+1}$ is drawn i.i.d from the same exponential distribution. $V^i_k(x^i_k)$ includes the network loss $r^i_k$ among diverse agents in the neighborhood of agent $i$, which signifies the communication among them for each iteration. For the update at the iteration $k$, $\mathcal{T}_{\rho_i,\beta_i}$ takes as input the cumulative loss up to $k$ and the random perturbation $\sigma^i_{k+1}$ and then predicts $x^i_{k+1}$. Intuitively, the next approximate best decision is attained in hindsight by perturbing the cumulative loss. Next, we will show how such a decision-making process benefits from the perturbation for stability.
\subsection{Sublinear Regret for DINOCO}
In this section, we analyze the regret bound for DINOCO and shed light on how the composite regret is leveraged for the setting of distributed online non-convex learning. To ease the complications of the analysis, we introduce the $(\rho,\beta)$-\textit{aggregate approximate optimization oracle} for the network such that
\begin{equation}
    V(\mathbf{x}_*) - \sigma^\top \mathbf{x}_*\leq\textnormal{inf}_{\mathbf{x}\in\mathcal{X}^N}V(\mathbf{x})-\sigma^\top\mathbf{x}+(\rho+\beta\|\sigma\|_1),
\end{equation}
where $\rho = \textnormal{max}\{\rho_1,\rho_2,...,\rho_N\}$ and $\beta = \textnormal{max}\{\beta_1,\beta_2,...,\beta_N\}$. Each coordinate of $\sigma$ corresponds to the random perturbation of an agent, which is sample from the same exponential distribution in the i.i.d. manner.
The aggregate in this context is due to the agents in the network. 
A key lemma for upper bounding the expected composite regret is first given.
\begin{lem}\label{lem_9}
Let all assumptions hold. The expected composite regret of the DINOCO is upper bounded as
\begin{equation}\label{eq_75}
    \mathbb{E}\Bigg[\sum_{k=1}^KV_k(\mathbf{x}_k)-\sum_{k=1}^KV_k(\mathbf{x}_*)\Bigg]\leq L\sum_{k=1}^K\mathbb{E}[\|\mathbf{x}_{k+1}-\mathbf{x}_k\|_1]+\frac{N(\beta K+\sqrt{N}D)}{\eta}+\rho K,
\end{equation}
where $\mathbf{x}_*=\underset{\mathbf{x}\in\mathcal{X}^N}{\textnormal{inf}}\sum_{k=1}^KV_k(\mathbf{x})$.
\end{lem}
As Lemma~\ref{lem_9} shows, the upper bound is determined by the stability of decisions, the approximation error by the aggregate approximate optimization oracle, and the diameter of the decision set. Another suggestion from Lemma~\ref{lem_9} is that the number of agents can impact the upper bound, which is evident in the second term on the right-hand side of Eq.~\ref{eq_75}. However, the impact could be reduced by choosing appropriately $\eta$ and $\beta$. We defer the analysis after presenting the main theorem. This term also shows the tradeoff between the stability and the quality of predictions. Although more randomness can better stabilize the decision-making process, the prediction may become worse to cause the worse regret when $\eta$ decreases.
Additionally, if substituting $L\approx\hat{G}(1+\frac{\gamma}{1-\lambda}), \gamma> 0$ (please see Appendix~\ref{app_c}) into the upper bound, we can obtain the quantitative relationship between the regret bound and the network topology, which provides us a way to investigate regrets with different topologies. To get the main theorem, we have to show the first term on the right-hand side in Eq.~\ref{eq_75} to be bounded above. Therefore, we provide details in the Appendix~\ref{app_c}. In this context, we have adopted  $\textnormal{SC-regret}_K$ to show the regret bound, and the application of $\textnormal{DC-regret}_K$ for non-convex losses is left as future work.
We now state the main theorem for DINOCO in the setting of distributed online non-convex learning. 
\begin{thm}\label{theorem_4}
Suppose that the offline optimization oracle used by DINOCO is a ($\rho,\beta$)-aggregate approximate. Also, the random perturbation $\sigma$ is assumed to be sampled from an exponential distribution $\textnormal{exp}(\eta)$ with a constant rate parameter $\eta$. Then DINOCO can achieve the expected regret bound where $\mathbf{x}_*=\underset{\mathbf{x}\in\mathcal{X}^N}{\textnormal{inf}}\sum_{k=1}^KV_k(\mathbf{x})$:
\begin{equation}
\begin{split}
    \mathbb{E}\Bigg[\sum_{k=1}^KV_k(\mathbf{x}_k)-\sum_{k=1}^KV_k(\mathbf{x}_*)\Bigg]&\leq\mathcal{O}\Bigg(K\eta N^{\frac{5}{2}}D\hat{G}^2\Bigg(1+\frac{\gamma}{1-\lambda}\Bigg)^2+\frac{N(\beta K+\sqrt{N}D)}{\eta}+\rho K\\&+K\beta N\hat{G}\Bigg(1+\frac{\gamma}{1-\lambda}\Bigg)\Bigg),
\end{split}
\end{equation}
\end{thm}
The proof is provided in Appendix~\ref{app_c}. We omit some constants in the regret bound for the expected composite regret in Theorem~\ref{theorem_4}. Unlike the previous analysis given for pseudo-convex losses, the regret bound for non-convex losses is more sophisticated. Its growth order relies on some parameters, including the rate parameter $\eta$ and the parameters associated with the offline optimization oracle. Critical constants such as the number of agents and the spectral gap of the network topology also affect the regret bound. The regret bound increases as the number of agents $N$ increases. When $N$ has a similar value as the number of iterations $K$, the regret bound becomes trivial, but in practice, $N$ remains finite and a smaller number compared to $K$. We now look into the impact of topology. With a dense topology, the spectral gap $1-\lambda$ is relatively larger, leading to a smaller regret bound, while a sparse topology yields a larger regret bound. Such a property has been well studied in distributed offline non-convex optimization and still holds for the distributed online non-convex optimization. Next, we present a corollary to show the sublinear regret bound by explicitly setting $\eta$, $\rho$, and $\beta$.

\begin{corol}\label{coro_2}
Suppose that the offline optimization oracle used by DINOCO is ($\rho,\beta$)-aggregate approximate, where $\rho=\mathcal{O}(\frac{1}{\sqrt{K}})$ and $\beta=\mathcal{O}(\frac{1}{K})$. Also, the random perturbation $\sigma$ is assumed to be sampled from an exponential distribution $\textnormal{exp}(\eta)$ with a constant rate parameter $\eta=\mathcal{O}(\frac{1}{\sqrt{K}})$. Then DINOCO can achieve the expected regret bound
\begin{equation}\label{eq_95}
    \mathbb{E}\Bigg[\sum_{k=1}^KV_k(\mathbf{x}_k)-\sum_{k=1}^KV_k(\mathbf{x}_*)\Bigg]\leq\mathcal{O}\Bigg(\sqrt{K}N^{\frac{5}{2}}\Bigg(1+\frac{\gamma}{1-\lambda}\Bigg)^2+\sqrt{K}N^{\frac{3}{2}}+\sqrt{K}\Bigg),
\end{equation}
where $\mathbf{x}_*=\underset{\mathbf{x}\in\mathcal{X}^N}{\textnormal{inf}}\sum_{k=1}^KV_k(\mathbf{x})$.
\end{corol}
The proof of Corollary~\ref{coro_2} is an immediate consequence of Theorem~\ref{theorem_4} when substituting $\eta$, $\rho$, and $\beta$ into the regret bound. It clearly shows that the regret bound for the DINOCO with the expected composite regret is sublinear and closely related to the number of agents and the topology of the network. When $N=1$ such that the term $\frac{\gamma}{1-\lambda}$ disappears and the setting is centralized, the regret bound maintains $\mathcal{O}(\sqrt{K})$. Hence we can recover the regret bound shown in~\citet{suggala2019online}. When $N$ is large, one can observe that the \textit{first} term on the right-hand side of Eq.~\ref{eq_95} dominates the regret bound. 
By specifically defining $\eta$ and $\beta$ as a function of $N$, the order of $N$ in the regret bound can be reduced. Let $\eta = \mathcal{O}(\frac{1}{\sqrt{NK}})$ and $\beta = \mathcal{O}(\frac{1}{NK})$. We then have
\begin{equation}\label{eq_96}
    \mathbb{E}\Bigg[\sum_{k=1}^KV_k(\mathbf{x}_k)-\sum_{k=1}^KV_k(\mathbf{x}_*)\Bigg]\leq\mathcal{O}\Bigg(\sqrt{K}N^{2}\Bigg(1+\frac{\gamma}{1-\lambda}\Bigg)^2+\sqrt{K}N^{2}+\sqrt{K}\Bigg),
\end{equation}
which suggests that the regret bound is approximately reduced by an order of $\sqrt{N}$. In summary, though the extension of online non-convex optimization algorithm from centralized to distributed setting is straightforward, such an extension is still critical as rare results have been reported and the analysis involving diverse topologies is non-trivial.
\section{Conclusions}\label{conclusion}
This paper develops a novel regret that extends the regular static and dynamic regrets to account for the network loss in distributed online learning. With the composite regret, we propose algorithms for different classes of losses. For pseudo-convex losses, the proposed algorithm allows achieving the sublinear regret, which is the first time for the distributed setting. For general non-convex losses, based on the recent development, we have found that no deterministic algorithm can achieve sublinear regret. By introducing an offline optimization oracle, the proposed DINOCO using the composite regret can achieve the best sublinear regret, which sheds light on the regret bound in distributed online non-convex learning. Additionally, for strongly convex and non-strongly convex losses, the regrets are shown to match the best performance of the state-of-the-art.




\bibliography{alt2021-sample}

\begin{thebibliography}{39}
\providecommand{\natexlab}[1]{#1}
\providecommand{\url}[1]{\texttt{#1}}
\expandafter\ifx\csname urlstyle\endcsname\relax
  \providecommand{\doi}[1]{doi: #1}\else
  \providecommand{\doi}{doi: \begingroup \urlstyle{rm}\Url}\fi

\bibitem[Agarwal et~al.(2018)Agarwal, Gonen, and Hazan]{agarwal2018learning}
Naman Agarwal, Alon Gonen, and Elad Hazan.
\newblock Learning in non-convex games with an optimization oracle.
\newblock \emph{arXiv preprint arXiv:1810.07362}, 2018.

\bibitem[Akbari et~al.(2015)Akbari, Gharesifard, and
  Linder]{akbari2015distributed}
Mohammad Akbari, Bahman Gharesifard, and Tam{\'a}s Linder.
\newblock Distributed online convex optimization on time-varying directed
  graphs.
\newblock \emph{IEEE Transactions on Control of Network Systems}, 4\penalty0
  (3):\penalty0 417--428, 2015.

\bibitem[Aydore et~al.(2018)Aydore, Dicker, and Foster]{aydore2018local}
Sergul Aydore, Lee Dicker, and Dean Foster.
\newblock A local regret in nonconvex online learning.
\newblock \emph{arXiv preprint arXiv:1811.05095}, 2018.

\bibitem[Bastianello et~al.(2020)Bastianello, Ajalloeian, and
  Dall'Anese]{bastianello2020distributed}
Nicola Bastianello, Amirhossein Ajalloeian, and Emiliano Dall'Anese.
\newblock Distributed and inexact proximal gradient method for online convex
  optimization.
\newblock \emph{arXiv preprint arXiv:2001.00870}, 2020.

\bibitem[Belmega et~al.(2018)Belmega, Mertikopoulos, Negrel, and
  Sanguinetti]{belmega2018online}
E~Veronica Belmega, Panayotis Mertikopoulos, Romain Negrel, and Luca
  Sanguinetti.
\newblock Online convex optimization and no-regret learning: Algorithms,
  guarantees and applications.
\newblock \emph{arXiv preprint arXiv:1804.04529}, 2018.

\bibitem[Berahas et~al.(2018)Berahas, Bollapragada, Keskar, and
  Wei]{berahas2018balancing}
Albert~S Berahas, Raghu Bollapragada, Nitish~Shirish Keskar, and Ermin Wei.
\newblock Balancing communication and computation in distributed optimization.
\newblock \emph{IEEE Transactions on Automatic Control}, 64\penalty0
  (8):\penalty0 3141--3155, 2018.

\bibitem[Besbes et~al.(2015)Besbes, Gur, and Zeevi]{besbes2015non}
Omar Besbes, Yonatan Gur, and Assaf Zeevi.
\newblock Non-stationary stochastic optimization.
\newblock \emph{Operations research}, 63\penalty0 (5):\penalty0 1227--1244,
  2015.

\bibitem[Cao and Ba{\c{s}}ar(2021)]{cao2021decentralized}
Xuanyu Cao and Tamer Ba{\c{s}}ar.
\newblock Decentralized online convex optimization based on signs of relative
  states.
\newblock \emph{Automatica}, 129:\penalty0 109676, 2021.

\bibitem[Cao and Basar(2021)]{cao2021decentralizedb}
Xuanyu Cao and Tamer Basar.
\newblock Decentralized online convex optimization with feedback delays.
\newblock \emph{IEEE Transactions on Automatic Control}, 2021.

\bibitem[Chiang et~al.(2012)Chiang, Yang, Lee, Mahdavi, Lu, Jin, and
  Zhu]{chiang2012online}
Chao-Kai Chiang, Tianbao Yang, Chia-Jung Lee, Mehrdad Mahdavi, Chi-Jen Lu, Rong
  Jin, and Shenghuo Zhu.
\newblock Online optimization with gradual variations.
\newblock In \emph{Conference on Learning Theory}, pages 6--1, 2012.

\bibitem[Eshraghi and Liang(2020)]{eshraghi2020distributed}
Nima Eshraghi and Ben Liang.
\newblock Distributed online optimization over a heterogeneous network with
  any-batch mirror descent.
\newblock In \emph{International Conference on Machine Learning}, pages
  2933--2942. PMLR, 2020.

\bibitem[Gao et~al.(2018)Gao, Li, and Zhang]{gao2018online}
Xiand Gao, Xiaobo Li, and Shuzhong Zhang.
\newblock Online learning with non-convex losses and non-stationary regret.
\newblock In \emph{International Conference on Artificial Intelligence and
  Statistics}, pages 235--243, 2018.

\bibitem[Hazan et~al.(2017)Hazan, Singh, and Zhang]{hazan2017efficient}
Elad Hazan, Karan Singh, and Cyril Zhang.
\newblock Efficient regret minimization in non-convex games.
\newblock \emph{arXiv preprint arXiv:1708.00075}, 2017.

\bibitem[Hazan et~al.(2016)]{hazan2016introduction}
Elad Hazan et~al.
\newblock Introduction to online convex optimization.
\newblock \emph{Foundations and Trends{\textregistered} in Optimization},
  2\penalty0 (3-4):\penalty0 157--325, 2016.

\bibitem[Hosseini et~al.(2016)Hosseini, Chapman, and
  Mesbahi]{hosseini2016online}
Saghar Hosseini, Airlie Chapman, and Mehran Mesbahi.
\newblock Online distributed convex optimization on dynamic networks.
\newblock \emph{IEEE Transactions on Automatic Control}, 61\penalty0
  (11):\penalty0 3545--3550, 2016.

\bibitem[Jiang et~al.(2017)Jiang, Balu, Hegde, and
  Sarkar]{jiang2017collaborative}
Zhanhong Jiang, Aditya Balu, Chinmay Hegde, and Soumik Sarkar.
\newblock Collaborative deep learning in fixed topology networks.
\newblock In \emph{Advances in Neural Information Processing Systems}, pages
  5904--5914, 2017.

\bibitem[Lee and Zavlanos(2017)]{lee2017sublinear}
Soomin Lee and Michael~M Zavlanos.
\newblock On the sublinear regret of distributed primal-dual algorithms for
  online constrained optimization.
\newblock \emph{arXiv preprint arXiv:1705.11128}, 2017.

\bibitem[Li et~al.(2018)Li, Yi, and Xie]{li2018distributed}
Xiuxian Li, Xinlei Yi, and Lihua Xie.
\newblock Distributed online optimization for multi-agent networks with coupled
  inequality constraints.
\newblock \emph{arXiv preprint arXiv:1805.05573}, 2018.

\bibitem[Lian et~al.(2017)Lian, Zhang, Zhang, Hsieh, Zhang, and
  Liu]{lian2017can}
Xiangru Lian, Ce~Zhang, Huan Zhang, Cho-Jui Hsieh, Wei Zhang, and Ji~Liu.
\newblock Can decentralized algorithms outperform centralized algorithms? a
  case study for decentralized parallel stochastic gradient descent.
\newblock In \emph{Advances in Neural Information Processing Systems}, pages
  5330--5340, 2017.

\bibitem[Lu and Wang(2021)]{lu2021online}
Kaihong Lu and Long Wang.
\newblock Online distributed optimization with nonconvex objective functions:
  Sublinearity of first-order optimality condition-based regret.
\newblock \emph{IEEE Transactions on Automatic Control}, 2021.

\bibitem[Meng et~al.(2021)Meng, Li, Hong, Chen, and
  Wang]{meng2021decentralized}
Min Meng, Xiuxian Li, Yiguang Hong, Jie Chen, and Long Wang.
\newblock Decentralized online learning for noncooperative games in dynamic
  environments.
\newblock \emph{arXiv preprint arXiv:2105.06200}, 2021.

\bibitem[Meng et~al.(2022)Meng, Li, and Chen]{meng2022decentralizedb}
Min Meng, Xiuxian Li, and Jie Chen.
\newblock Decentralized nash equilibria learning for online game with bandit
  feedback.
\newblock \emph{arXiv preprint arXiv:2204.09467}, 2022.

\bibitem[Nazari et~al.(2019)Nazari, Tarzanagh, and
  Michailidis]{nazari2019dadam}
Parvin Nazari, Davoud~Ataee Tarzanagh, and George Michailidis.
\newblock Dadam: A consensus-based distributed adaptive gradient method for
  online optimization.
\newblock \emph{arXiv preprint arXiv:1901.09109}, 2019.

\bibitem[Nedi{\'c} and Olshevsky(2014)]{nedic2014distributed}
Angelia Nedi{\'c} and Alex Olshevsky.
\newblock Distributed optimization over time-varying directed graphs.
\newblock \emph{IEEE Transactions on Automatic Control}, 60\penalty0
  (3):\penalty0 601--615, 2014.

\bibitem[Nedi{\'c} et~al.(2015)Nedi{\'c}, Lee, and
  Raginsky]{nedic2015decentralized}
Angelia Nedi{\'c}, Soomin Lee, and Maxim Raginsky.
\newblock Decentralized online optimization with global objectives and local
  communication.
\newblock In \emph{2015 American Control Conference (ACC)}, pages 4497--4503.
  IEEE, 2015.

\bibitem[Nesterov(2013)]{nesterov2013introductory}
Yurii Nesterov.
\newblock \emph{Introductory lectures on convex optimization: A basic course},
  volume~87.
\newblock Springer Science \& Business Media, 2013.

\bibitem[Paternain et~al.(2019)Paternain, Lee, Zavlanos, and
  Ribeiro]{paternain2019distributed}
Santiago Paternain, Soomin Lee, Michael~M Zavlanos, and Alejandro Ribeiro.
\newblock Distributed constrained online learning.
\newblock \emph{arXiv preprint arXiv:1903.06310}, 2019.

\bibitem[Pu and Nedi{\'c}(2018)]{pu2018distributed}
Shi Pu and Angelia Nedi{\'c}.
\newblock A distributed stochastic gradient tracking method.
\newblock In \emph{2018 IEEE Conference on Decision and Control (CDC)}, pages
  963--968. IEEE, 2018.

\bibitem[Shahrampour and Jadbabaie(2017)]{shahrampour2017distributed}
Shahin Shahrampour and Ali Jadbabaie.
\newblock Distributed online optimization in dynamic environments using mirror
  descent.
\newblock \emph{IEEE Transactions on Automatic Control}, 63\penalty0
  (3):\penalty0 714--725, 2017.

\bibitem[Shalev-Shwartz et~al.(2012)]{shalev2012online}
Shai Shalev-Shwartz et~al.
\newblock Online learning and online convex optimization.
\newblock \emph{Foundations and Trends{\textregistered} in Machine Learning},
  4\penalty0 (2):\penalty0 107--194, 2012.

\bibitem[Sharma et~al.(2020)Sharma, Khanduri, Shen, Bucci~Jr, and
  Varshney]{sharma2020distributed}
Pranay Sharma, Prashant Khanduri, Lixin Shen, Donald~J Bucci~Jr, and Pramod~K
  Varshney.
\newblock On distributed online convex optimization with sublinear dynamic
  regret and fit.
\newblock \emph{arXiv preprint arXiv:2001.03166}, 2020.

\bibitem[Suggala and Netrapalli(2019)]{suggala2019online}
Arun~Sai Suggala and Praneeth Netrapalli.
\newblock Online non-convex learning: Following the perturbed leader is
  optimal.
\newblock \emph{arXiv preprint arXiv:1903.08110}, 2019.

\bibitem[Yi et~al.(2020)Yi, Li, Xie, and Johansson]{yi2020distributed}
Xinlei Yi, Xiuxian Li, Lihua Xie, and Karl~H Johansson.
\newblock Distributed online convex optimization with time-varying coupled
  inequality constraints.
\newblock \emph{IEEE Transactions on Signal Processing}, 68:\penalty0 731--746,
  2020.

\bibitem[Yuan et~al.(2019)Yuan, Proutiere, and Shi]{yuan2019distributed}
Deming Yuan, Alexandre Proutiere, and Guodong Shi.
\newblock Distributed online optimization with long-term constraints.
\newblock \emph{arXiv preprint arXiv:1912.09705}, 2019.

\bibitem[Yuan et~al.(2020)Yuan, Hong, Ho, and Xu]{yuan2020distributed}
Deming Yuan, Yiguang Hong, Daniel~WC Ho, and Shengyuan Xu.
\newblock Distributed mirror descent for online composite optimization.
\newblock \emph{IEEE Transactions on Automatic Control}, 2020.

\bibitem[Zeng and Yin(2018)]{zeng2018nonconvex}
Jinshan Zeng and Wotao Yin.
\newblock On nonconvex decentralized gradient descent.
\newblock \emph{IEEE Transactions on signal processing}, 66\penalty0
  (11):\penalty0 2834--2848, 2018.

\bibitem[Zhang et~al.(2019)Zhang, Ravier, Zavlanos, and
  Tarokh]{zhang2019distributed}
Yan Zhang, Robert~J Ravier, Michael~M Zavlanos, and Vahid Tarokh.
\newblock A distributed online convex optimization algorithm with improved
  dynamic regret.
\newblock \emph{arXiv preprint arXiv:1911.05050}, 2019.

\bibitem[Zhao et~al.(2019)Zhao, Yu, Zhao, Tang, Qiu, and
  Liu]{zhao2019decentralized}
Yawei Zhao, Chen Yu, Peilin Zhao, Hanlin Tang, Shuang Qiu, and Ji~Liu.
\newblock Decentralized online learning: Take benefits from others' data
  without sharing your own to track global trend.
\newblock \emph{arXiv preprint arXiv:1901.10593}, 2019.

\bibitem[Zinkevich(2003)]{zinkevich2003online}
Martin Zinkevich.
\newblock Online convex programming and generalized infinitesimal gradient
  ascent.
\newblock In \emph{Proceedings of the 20th international conference on machine
  learning (icml-03)}, pages 928--936, 2003.

\end{thebibliography}
\newpage
\appendix
\section{Analysis for CONGD}\label{app_a}
We provide the detailed analysis for CONGD.
\subsection{CONGD for Pseudo-convex Losses}
Recalling \[\mathbf{y}_{k+1} = 
    \Pi\mathbf{x}_k - \alpha_k\frac{\nabla \mathcal{F}_k(\mathbf{x}_k)}{\|\nabla \mathcal{F}_k(\mathbf{x}_k)\|},\]
we then give the upper bound for the consensus among agents that can be extended from Lemma~\ref{consensus_lem} of the OLGD. Thus, we have
\begin{lem}\label{lem_5}
Let Assumption~\ref{assump_1} hold. The sequence $\{x^i_k\}$ generated by the CONGD satisfies the following relationship
\begin{equation}
    \|x^i_k-\hat{x}_k\|\leq \sum_{s=1}^{k-1}\alpha_s\lambda^{k-1-s}.
\end{equation}
\end{lem}
When $\alpha_k=\alpha$ for all $k$, one immediate consequence from Lemma~\ref{lem_5} is that $\|x^i_k-\hat{x}_k\|\leq\frac{\alpha}{1-\lambda}$ such that we have $\frac{1}{\alpha}\|(\mathbf{I}_N-\Pi)\mathbf{x}_k\|\leq \frac{2}{1-\lambda}$. Intuitively, this result can be obtained by extending from Lemma~\ref{coro_1} as in the update law of CONGD it is the normalized gradient instead of the gradient.
We introduce the definition of pseudo-convexity for any differentiable function $f$.
\begin{defn}{\citep{nesterov2013introductory}}\label{pseudo-convex}
If a differentiable function $f(\cdot)$ is Lipschitz continuous and pseudo-convex, then for all $x, y\in\mathbb{R}^d$ with $f(x)\geq f(y)$, we have
\begin{equation}\label{ps-cov}
    f(x)-f(y)\leq M\frac{\nabla f(x)^T(x-y)}{\|\nabla f(x)\|},
\end{equation}
where $M$ is the Lipschitz constant. If $\nabla f(x)=0$, then $\frac{\nabla f(x)}{\|\nabla f(x)\|}=0$.
\end{defn}
According to Definition~\ref{pseudo-convex}, it is immediately obtained that for any constant $M'$ if $M'\geq M$, Eq.~\ref{ps-cov} still holds true. We then have that with the constant step size, if $V_k(\cdot)$ satisfies the pseudo-convexity, 
\begin{equation}
    V_k(\mathbf{x})-V_k(\mathbf{y})\leq C\frac{\nabla V_k(\mathbf{x})^T(\mathbf{x}-\mathbf{y})}{\|\nabla V_k(\mathbf{x})\|},
\end{equation}
for all $\mathbf{x}, \mathbf{y}\in\mathcal{X}$ with $V_k(\mathbf{x})\geq V_k(\mathbf{y})$, for all $k$.
Recall Eq.~\ref{congd_lya} as follows
\begin{equation}
    \mathbf{y}_{k+1} = 
    \mathbf{x}_k - \alpha_k\Bigg(\frac{\nabla \mathcal{F}_k(\mathbf{x}_k)}{\|\nabla \mathcal{F}_k(\mathbf{x}_k)\|}+\hat{\alpha}_k\frac{(\mathbf{I}_N-\Pi)\mathbf{x}_k}{\|(\mathbf{I}_N-\Pi)\mathbf{x}_k\|}\Bigg),
\end{equation}
which is the rewritten update rule for the CONGD. It can be observed that if we directly apply the composite regret to the CONGD, the equivalent gradient here can not match the exact gradient $\nabla V_k(\mathbf{x}_k)$, due to the normalization for $\nabla\mathcal{F}_k(\mathbf{x}_k)$ and $(\mathbf{I}_N-\Pi)\mathbf{x}_k$. Since there is no normalization in the composite regret. However, in this context, we leverage Definition~\ref{pseudo-convex} to separately study for $\mathcal{F}_k(\mathbf{x}_k)$ and $c\mathbf{x}^T(\mathbf{I}_N-\Pi)\mathbf{x}_k$ and then combine them together to form an equivalent composite regret, which surprisingly is exactly the same as $\textnormal{DC-regret}_K$. We will also show that there exists a constant that can lead us to obtain the equivalent Lipschitz constant for $\frac{\nabla \mathcal{F}_k(\mathbf{x}_k)}{\|\nabla \mathcal{F}_k(\mathbf{x}_k)\|}+\hat{\alpha}_k\frac{(\mathbf{I}_N-\Pi)\mathbf{x}_k}{\|(\mathbf{I}_N-\Pi)\mathbf{x}_k\|}$. The following lemma summarizes such a result. 
\begin{lem}\label{lem_6}
Let Assumption~\ref{assump_1} hold. There exists $\zeta> 0$ such that 
\begin{equation}\label{eq_57}
    V_k(\mathbf{x}_k)-V(\mathbf{x}_{*,k})\leq \zeta \Bigg(\frac{\nabla \mathcal{F}_k(\mathbf{x}_k)}{\|\nabla \mathcal{F}_k(\mathbf{x}_k)\|}+\hat{\alpha}_k\frac{(\mathbf{I}_N-\Pi)\mathbf{x}_k}{\|(\mathbf{I}_N-\Pi)\mathbf{x}_k\|}\Bigg)^T(\mathbf{x}_k-\mathbf{x}_{*,k}).
\end{equation}
\end{lem}

Intuitively, Lemma~\ref{lem_6} states that $V_k(\mathbf{x})$ for all $k$ satisfies the pseudo-convexity if $\mathcal{F}_k(\mathbf{x})$ is pseudo-convex. This is a key technical lemma to facilitate the analysis for the CONGD for pseudo-convex losses. We now present the proof for Theorem~\ref{theorem_3} in the following.
\section{Proof of Theorem~\ref{theorem_3}}\label{app_b}
\subsection{Proof for Lemma~\ref{lem_6}}
\begin{proof}
Since $\mathcal{F}_k(\mathbf{x}_k)$ is differentiable and pseudo-convex, applying Definition~\ref{pseudo-convex} to the pair of $\mathbf{x}_k$ and $\mathbf{x}_{*,k}$, the following relationship can be obtained
\begin{equation}\label{eq_58}
    \mathcal{F}_k(\mathbf{x}_k)-\mathcal{F}_k(\mathbf{x}_{*,k})\leq G\frac{\nabla \mathcal{F}_k(\mathbf{x}_k)^T(\mathbf{x}_k-\mathbf{x}_{*,k})}{\|\nabla \mathcal{F}(\mathbf{x}_k)\|}.
\end{equation}
Let $r(\mathbf{x})=c\mathbf{x}^T(\mathbf{I}_N-\Pi)\mathbf{x}$, which is convex function at any $\mathbf{x}$ because $\mathbf{I}_N-\Pi$ is positive semidefinite. Hence, analogously, applying Definition~\ref{pseudo-convex} to $r(\mathbf{x})$ yields
\begin{equation}\label{eq_59}
    r_k(\mathbf{x}_k)-r_k(\mathbf{x}_{*,k})\leq \frac{2}{1-\lambda}\frac{\mathbf{x}_k^T(\mathbf{I}_N-\Pi)^T(\mathbf{x}_k-\mathbf{x}_{*,k})}{\|(\mathbf{I}_N-\Pi)\mathbf{x}_k\|},
\end{equation}
with $c=\frac{1}{2\alpha}$ and $\alpha_k=\alpha, \forall k$. 

Since $V_k(\mathbf{x}_k)-V_k(\mathbf{x}_{*,k}):=\mathcal{F}_k(\mathbf{x}_k)+r_k(\mathbf{x}_k)-(\mathcal{F}_k(\mathbf{x}_{*,k})+r_k(\mathbf{x}_{*,k}))$, the following inequality is obtained
\begin{equation}
    V_k(\mathbf{x}_k)-V_k(\mathbf{x}_{*,k})\leq G\frac{\nabla \mathcal{F}_k(\mathbf{x}_k)^T(\mathbf{x}_k-\mathbf{x}_{*,k})}{\|\nabla \mathcal{F}(\mathbf{x}_k)\|}+\frac{2}{1-\lambda}\frac{\mathbf{x_k}^T(\mathbf{I}_N-\Pi)^T(\mathbf{x_k}-\mathbf{x}_{*,k})}{\|(\mathbf{I}_N-\Pi)\mathbf{x}_k\|}.
\end{equation}
However, to arrive at the conclusion, we need to find out a constant $\zeta> 0$ that enables the following inequality to hold
\begin{equation}\label{eq_61}
    \begin{split}
        G\frac{\nabla \mathcal{F}_k(\mathbf{x}_k)^T(\mathbf{x}_k-\mathbf{x}_{*,k})}{\|\nabla \mathcal{F}(\mathbf{x}_k)\|}&+\frac{2}{1-\lambda}\frac{\mathbf{x}_k^T(\mathbf{I}_N-\Pi)^T(\mathbf{x}_k-\mathbf{x}_{*,k})}{\|(\mathbf{I}_N-\Pi)\mathbf{x}_k\|}\\&\leq\zeta \Bigg(\frac{\nabla \mathcal{F}_k(\mathbf{x}_k)}{\|\nabla \mathcal{F}_k(\mathbf{x}_k)\|}+\hat{\alpha}\frac{(\mathbf{I}_N-\Pi)\mathbf{x}_k}{\|(\mathbf{I}_N-\Pi)\mathbf{x}_k\|}\Bigg)^T(\mathbf{x}_k-\mathbf{x}_{*,k}).
    \end{split}
\end{equation}
Define $\zeta=\textnormal{max}\{G,b\}$, where $b$ satisfies $\frac{2}{1-\lambda}\leq b\hat{\alpha}$ and $0< b<\infty$. We justify the validity of $\frac{2}{1-\lambda}\leq b\hat{\alpha}$. Since the upper bound of $\hat{\alpha}$ is $\frac{2}{1-\lambda}<\infty $, there exists $0< b<\infty$ enabling $\frac{2}{1-\lambda}\leq b\hat{\alpha}$ to hold true.
We next show that when $\zeta$ satisfies the above condition, Eq.~\ref{eq_57} holds.

If $\zeta= G$, then we have
\begin{equation}
    G\frac{\nabla \mathcal{F}_k(\mathbf{x}_k)^T(\mathbf{x}_k-\mathbf{x}_{*,k})}{\|\nabla \mathcal{F}(\mathbf{x}_k)\|}\leq\zeta\frac{\nabla \mathcal{F}_k(\mathbf{x}_k)^T(\mathbf{x}_k-\mathbf{x}_{*,k})}{\|\nabla \mathcal{F}_k(\mathbf{x}_k)\|}.
\end{equation}
To guarantee Eq.~\ref{eq_61} to hold, we require
\begin{equation}
    \frac{2}{1-\lambda}\frac{\mathbf{x}_k^T(\mathbf{I}_N-\Pi)^T(\mathbf{x}_k-\mathbf{x}_{*,k})}{\|(\mathbf{I}_N-\Pi)\mathbf{x}_k\|}\leq b\hat{\alpha}\frac{\mathbf{x}_k^T(\mathbf{I}_N-\Pi)^T(\mathbf{x}_k-\mathbf{x}_{*,k})}{\|(\mathbf{I}_N-\Pi)\mathbf{x}_k\|}.
\end{equation}
Since now $b\leq G$, we have $\frac{2}{1-\lambda}\leq b\hat{\alpha}\leq G\hat{\alpha}$. Thus, Eq.~\ref{eq_61} holds. Similarly, when $\zeta=b$, we arrive at the same conclusion.
\end{proof}
\subsection{Proof for Theorem~\ref{theorem_3}}
\begin{proof}
Let $\mathbf{x}_{*,k}, k=1,2,...,K$ be the sequence of optimal solutions satisfying Eq.~\ref{eq_1}. Then we have:
\begin{equation}
\begin{split}
    \|\mathbf{x}_{k+1}-\mathbf{x}_{*,k+1}\|^2&=\|\mathbf{x}_{k+1}-\mathbf{x}_{*,k}\|^2+\|\mathbf{x}_{*,k}-\mathbf{x}_{*,k+1}\|^2\\&+2(\mathbf{x}_{k+1}-\mathbf{x}_{*,k})^T(\mathbf{x}_{*,k}-\mathbf{x}_{*,k+1})\\&\leq\Bigg\|\mathbf{P}_{\mathcal{X}^N}\Bigg(\mathbf{x}_k-\alpha\Bigg(\frac{\nabla\mathcal{F}(\mathbf{x}_k)}{\|\nabla\mathcal{F}(\mathbf{x}_k)\|}+\hat{\alpha}\frac{(\mathbf{I}_N-\Pi)\mathbf{x}_k}{\|(\mathbf{I}_N-\Pi)\mathbf{x}_k\|}\Bigg)\Bigg)-\mathbf{x}_{*,k}\Bigg\|^2\\&+3\sqrt{N}D\|\mathbf{x}_{*,k}-\mathbf{x}_{*,k+1}\|\\&\leq\Bigg\|\mathbf{x}_k-\alpha\Bigg(\frac{\nabla\mathcal{F}(\mathbf{x}_k)}{\|\nabla\mathcal{F}(\mathbf{x}_k)\|}+\hat{\alpha}\frac{(\mathbf{I}_N-\Pi)\mathbf{x}_k}{\|(\mathbf{I}_N-\Pi)\mathbf{x}_k\|}\Bigg)-\mathbf{x}_{*,k}\Bigg\|^2\\&+3\sqrt{N}D\|\mathbf{x}_{*,k}-\mathbf{x}_{*,k+1}\|\\&=\|\mathbf{x}_{k}-\mathbf{x}_{*,k}\|^2+\alpha^2\Bigg\|\frac{\nabla\mathcal{F}(\mathbf{x}_k)}{\|\nabla\mathcal{F}(\mathbf{x}_k)\|}+\hat{\alpha}\frac{(\mathbf{I}_N-\Pi)\mathbf{x}_k}{\|(\mathbf{I}_N-\Pi)\mathbf{x}_k\|}\Bigg\|^2\\&-2\alpha\Bigg(\frac{\nabla\mathcal{F}(\mathbf{x}_k)}{\|\nabla\mathcal{F}(\mathbf{x}_k)\|}+\hat{\alpha}\frac{(\mathbf{I}_N-\Pi)\mathbf{x}_k}{\|(\mathbf{I}_N-\Pi)\mathbf{x}_k\|}\Bigg)^T(\mathbf{x}_k-\mathbf{x}_{*,k})\\&+3\sqrt{N}D\|\mathbf{x}_{*,k}-\mathbf{x}_{*,k+1}\|.
\end{split}
\end{equation}
By rearranging terms and multiplying $\zeta$ on both sides we have
\begin{equation}
\begin{split}
    &\zeta\Bigg(\frac{\nabla\mathcal{F}(\mathbf{x}_k)}{\|\nabla\mathcal{F}(\mathbf{x}_k)\|}+\hat{\alpha}\frac{(\mathbf{I}_N-\Pi)\mathbf{x}_k}{\|(\mathbf{I}_N-\Pi)\mathbf{x}_k\|}\Bigg)^T(\mathbf{x}_k-\mathbf{x}_{*,k})\\&\leq\frac{\zeta}{2\alpha}\Bigg(\|\mathbf{x}_{k}-\mathbf{x}_{*,k}\|^2-\|\mathbf{x}_{k+1}-\mathbf{x}_{*,k+1}\|^2+\alpha^2\Bigg\|\frac{\nabla\mathcal{F}(\mathbf{x}_k)}{\|\nabla\mathcal{F}(\mathbf{x}_k)\|}+\hat{\alpha}\frac{(\mathbf{I}_N-\Pi)\mathbf{x}_k}{\|(\mathbf{I}_N-\Pi)\mathbf{x}_k\|}\Bigg\|^2\Bigg)\\&+\frac{3\sqrt{N}D\zeta}{2\alpha}\|\mathbf{x}_{*,k}-\mathbf{x}_{*,k+1}\|.
\end{split}
\end{equation}
Summing the above inequality from $k=1,2,...,K$, we have
\begin{equation}
\begin{split}
    &\sum_{k=1}^K(V_k(\mathbf{x}_k)-V_k(\mathbf{x}_{*,k}))\leq \frac{\zeta}{2\alpha}(\|\mathbf{x}_1-\mathbf{x}_{*,1}\|^2-\|\mathbf{x}_{K+1}-\mathbf{x}_{*,K+1}\|^2)\\&+\frac{\zeta\alpha}{2}\sum_{k=1}^K\Bigg\|\frac{\nabla\mathcal{F}(\mathbf{x}_k)}{\|\nabla\mathcal{F}(\mathbf{x}_k)\|}+\hat{\alpha}\frac{(\mathbf{I}_N-\Pi)\mathbf{x}_k}{\|(\mathbf{I}_N-\Pi)\mathbf{x}_k\|}\Bigg\|^2\\&+\frac{3\sqrt{N}D\zeta}{2\alpha}\sum^K_{k=1}\|\mathbf{x}_{*,k}-\mathbf{x}_{*,k+1}\|\\&\leq\frac{\zeta N D^2}{2\alpha}+\frac{\zeta\alpha}{2}K\Bigg(1+\frac{2}{1-\lambda}\Bigg)^2+\frac{3\sqrt{N}D\zeta}{2\alpha}P_K
\end{split}
\end{equation}
The first inequality follows from Lemma~\ref{lem_6}. The first term of the second inequality follows from $-\|\mathbf{x}_{K+1}-\mathbf{x}_{*,K+1}\|^2\leq 0$, the second term follows from $\Bigg\|\frac{\nabla\mathcal{F}(\mathbf{x}_k)}{\|\nabla\mathcal{F}(\mathbf{x}_k)\|}+\hat{\alpha}\frac{(\mathbf{I}_N-\Pi)\mathbf{x}_k}{\|(\mathbf{I}_N-\Pi)\mathbf{x}_k\|}\Bigg\|\leq 1+\frac{2}{1-\lambda}$, and the third term follows from Eq.~\ref{eq_1}. Define the step size \[\alpha=\sqrt{\frac{\sqrt{N}D(3P_K+\sqrt{N}D)}{K}}\] such that
\begin{equation}
\begin{split}
    &\sum_{k=1}^K(V_k(\mathbf{x}_k)-V_k(\mathbf{x}_{*,k}))\leq\frac{\zeta\sqrt{K}}{\sqrt{\sqrt{N}D(3P_K+\sqrt{N}D)}}+\Bigg((ND^2+3\sqrt{N}DP_K)\Bigg(1+\frac{1}{1-\lambda}\Bigg)\Bigg)\\&=\zeta\Bigg(1+\frac{1}{1-\lambda}\Bigg)\sqrt{K(ND^2+3\sqrt{N}DP_K)},
\end{split}
\end{equation}
which completes the proof.
\end{proof}

\section{Proof of Theorem~\ref{theorem_4}}\label{app_c}
A few technical lemmas are presented in the following to help characterize Theorem~\ref{theorem_4}.
First, a lemma is presented to introduce some nice properties of an exponential distribution. 
\begin{lem}\citep{agarwal2018learning}\label{lem_7}
Let $X$ be an exponential random variable with a rate parameter $\eta$. Then the following properties hold: a) for any $o\in\mathbb{R}$, $\mathbb{P}(X\geq o)=\textnormal{exp}(-\eta o)$. b) For any $o,p\in\mathbb{R}, \mathbb{P}(X\geq p+o|X\geq p)=\mathbb{P}(X\geq o)$.
\end{lem}
Part (b) in Lemma~\ref{lem_7} is known as \textit{memorylessness} of an exponential distribution. Due to the combined loss between $\mathcal{F}$ and $r$ in $V$, we need to derive the equivalent Lipschitz constant for $V$. One simple way is to directly impose an assumption for $V$ by introducing a positive constant, but in a more explicit sense, we present a close form involving the Lipschitz constant of $\mathcal{F}$.
\begin{lem}\label{lem_8}
There exists a constant $\gamma> 0$ such that for all $\mathbf{y},\mathbf{x}\in\mathcal{X}^N$
\begin{equation}\label{eq_74}
    |V(\mathbf{y})-V(\mathbf{x})|\leq L\|\mathbf{y}-\mathbf{x}\|_1,
\end{equation}
where $L\approx\hat{G}\Bigg(1+\frac{\gamma}{1-\lambda}\Bigg)$.
\end{lem}
One observation from Lemma~\ref{lem_8} is that the Lipschitz constant $L$ has a close form as the upper bound of $\nabla V$ for the CONGD, depending on the step size $\alpha$. In DINOCO, as discussed before, $\eta$ in fact plays a similar role as a step size in gradient-based algorithm, though it is a rate parameter for an exponential distribution. That is also the reason why an approximation is given for $L$ instead of an accurate expression. Since using the DINOCO in this context allows us to predict the decision without access to the gradients and even in the centralized setting, the direct application of FTPL to a convex loss has not yet reported~\citep{agarwal2018learning}. Given a sufficiently large $\gamma$, we can at least find an $L$ to satisfy Eq.~\ref{eq_74}.
\subsection{Proof for Lemma~\ref{lem_9}}
\begin{proof}
Recalling the static composite regret yields
\[\sum_{k=1}^KV_k(\mathbf{x}_k)-\sum_{k=1}^KV_k(\mathbf{x}_*)=\sum_{k=1}^K[V_k(\mathbf{x}_k)-V_k(\mathbf{x}_{k+1})]+\sum_{k=1}^K[V_k(\mathbf{x}_{k+1})-V_k(\mathbf{x}_*)].\]
Applying the Lispchitz continuity results in
\begin{equation}\label{eq_76}
    \sum_{k=1}^K[V_k(\mathbf{x}_k)-V_k(\mathbf{x}_*)]\leq L\sum_{k=1}^K\|\mathbf{x}_k-\mathbf{x}_{k+1}\|_1+\sum_{k=1}^K[V_k(\mathbf{x}_{k+1})-V_k(\mathbf{x}_*)].
\end{equation}
To arrive at Eq.~\ref{eq_75}, we need to upper bound the second term on the right hand side in Eq.~\ref{eq_76}. We next show by induction that $\sum_{k=1}^K[V_k(\mathbf{x}_{k+1})-V_k(\mathbf{x}_*)]\leq K(\rho+\beta\|\sigma\|_1) + \sigma^T(\mathbf{x}_2-\mathbf{x}_*)$.

Start with the base scenario where $K=1$. As $\mathbf{x}_2$ is the aggregate approximate optimizer of $V_1(\mathbf{x})-\sigma^T\mathbf{x}$, we have
\begin{equation}
    V_1(\mathbf{x}_2)-\sigma^T\mathbf{x}_2\leq\underset{\mathbf{x}\in\mathbf{X}^N}{\textnormal{inf}}V_1(\mathbf{x})-\sigma^T\mathbf{x}+(\rho+\beta\|\sigma\|_1)\leq V_1(\mathbf{x}_*)-\sigma^T\mathbf{x}_*+(\rho+\beta\|\sigma\|_1).
\end{equation}
The second inequality holds true for any $\mathbf{x}_*\in\mathcal{X}^N$. Consequently, we have $V_1(\mathbf{x}_2)-V_1(\mathbf{x}_*)\leq\sigma^T(\mathbf{x}_2-\mathbf{x}_*)$.

We now proceed to prove the induction step. Suppose that the claim holds for all $K\leq K_0$. We next show that it also holds for $K_0+1$.
\begin{equation}
    \begin{split}
        \sum^{K_0+1}_{k=1}V_k(\mathbf{x}_{k+1})&\leq\Bigg[\sum_{k=1}^{K_0}V_k(\mathbf{x}_{K_0+2})+\sigma^T(\mathbf{x}_2-\mathbf{x}_{K_0+2})+(\rho+\beta\|\sigma\|_1)K_0\Bigg]+V_{K_0+1}(\mathbf{x}_{K_0+2})\\&=\Bigg[\sum_{k=1}^{K_0+1}V_k(\mathbf{x}_{K_0+2})-\sigma^T\mathbf{x}_{K_0+2}\Bigg]+\sigma^T\mathbf{x}_2+(\rho+\beta\|\sigma\|_1)K_0\\&\leq\sum_{k=1}^{K_0+1}V_k(\mathbf{x}_*)+\sigma^T(\mathbf{x}_2-\mathbf{x}_*)+(\rho+\beta\|\sigma\|_1)(K_0+1).
    \end{split}
\end{equation}
The first inequality follows from that the claim holds for $K\leq K_0$ and the second inequality follows from the approximate optimality of $\mathbf{x}_{K_0+2}$. Therefore, we can get the upper bound for the expected composite regret
\begin{equation}
\begin{split}
    \mathbb{E}\Bigg[\sum_{k=1}^KV_k(\mathbf{x}_k)-\sum_{k=1}^KV_k(\mathbf{x}_*)\Bigg]&\leq L\sum_{k=1}^K\mathbb{E}[\|\mathbf{x}_{k+1}-\mathbf{x}_k\|_1]+\mathbb{E}[K(\rho+\beta\|\sigma\|_1) + \sigma^T(\mathbf{x}_2-\mathbf{x}_*)]\\&\leq L\sum_{k=1}^K\mathbb{E}[\|\mathbf{x}_{k+1}-\mathbf{x}_k\|_1]+\rho K+(\beta K+\sqrt{N}D)\Bigg(\sum_{i=1}^N\mathbb{E}[\sigma_i]\Bigg)\\&=L\sum_{k=1}^K\mathbb{E}[\|\mathbf{x}_{k+1}-\mathbf{x}_k\|_1]+\rho K+\frac{(\beta K+\sqrt{N}D)N}{\eta}.
\end{split}
\end{equation}
The second inequality follows from Assumption~\ref{assump_2} while the last equality follows from the property of the exponential distribution, $\mathbb{E}[\sigma_i]=\frac{1}{\eta_i}$. Then the proof is completed.
\end{proof}
\subsection{Proof for Lemma~\ref{lem_11}}
We introduce in this context two auxiliary technical lemmas and focus on the proof of the second one.
\begin{lem}\label{lem_10}
Let $\mathbf{x}_k(\sigma)$ be the prediction of the DINOCO in iteration $k$, with a random perturbation $\sigma$. Let $\mathbf{e}_i$ denote the $i$-th standard basis vector and $\mathbf{x}_{k,i}$ denote the $i$-th coordinate of $\mathbf{x}_k$. Then the following inequality holds for any $h> 0$
\begin{equation}
    \mathbf{x}_{k,i}(\sigma+h\mathbf{e}_i)\geq \mathbf{x}_{k,i}(\sigma) - \frac{2(\rho+\beta\|\sigma\|_1)}{h}-\beta.
\end{equation}
\end{lem}
The proof can directly follow from Lemma 5 in~\citep{suggala2019online}. Lemma~\ref{lem_10} states the monotonicity property for decisions played with the perturbation, which helps characterize another monotonicity required for the subsequent analysis of stability.
\begin{lem}\label{lem_11}
Let $\mathbf{x}_k(\sigma)$ be the prediction of FTPL2 in iteration $k$, with random perturbation $\sigma$. Let $\mathbf{e}_i$ denote the $i$-th standard basis vector and $\mathbf{x}_{k,i}$ denote the $i$-th coordinate of $\mathbf{x}_k$. Suppose that $\|\mathbf{x}_k(\sigma)-\mathbf{x}_{k+1}(\sigma)\|_1\leq b_1N|\mathbf{x}_{k,i}-\mathbf{x}_{k+1,i}(\sigma)|$. For $\hat{\sigma}=\sigma+b_2LN\mathbf{e}_i$, we have
\begin{equation}
\begin{split}
    \textnormal{min}(\mathbf{x}_{k,i}(\hat{\sigma}), \mathbf{x}_{k+1,i}(\hat{\sigma}))&\geq\textnormal{max}(\mathbf{x}_{k,i}(\sigma), \mathbf{x}_{k+1,i}(\sigma))-\frac{1}{b_1}|\mathbf{x}_{k,i}(\sigma)-\mathbf{x}_{k+1,i}(\sigma)|\\&-\frac{3(\rho+\beta\|\sigma\|_1)}{b_2LN}-\beta,
\end{split}
\end{equation}
where $b_1,b_2\geq 3$.
\end{lem}
\begin{proof}
Let $\xi(\sigma)=\rho+\beta\|\sigma\|_1$ be the approximation error of the aggregate approximate optimization oracle. When evaluating the approximate optimizer of $\mathbf{x}_k(\sigma)$, we have
\begin{equation}\label{eq_82}
\begin{split}
\sum_{l=1}^{k-1}V_l&(\mathbf{x}_k(\sigma))-\sigma^T\mathbf{x}_k(\sigma)+V_k(\mathbf{x}_k(\sigma))\\&\leq\sum_{l=1}^{k-1}V_l(\mathbf{x}_{k+1}(\sigma))-\sigma^T\mathbf{x}_{k+1}(\sigma)+V_k(\mathbf{x}_k(\sigma))+\xi(\sigma)\\&\leq\sum_{l=1}^{k-1}V_l(\mathbf{x}_{k+1}(\sigma))-\sigma^T\mathbf{x}_{k+1}(\sigma)+V_k(\mathbf{x}_k(\sigma))+V_k(\mathbf{x}_{k+1}(\sigma))-V_k(\mathbf{x}_{k+1}(\sigma))+\xi(\sigma)\\&\leq\sum_{l=1}^{k-1}V_l(\mathbf{x}_{k+1}(\sigma))-\sigma^T\mathbf{x}_{k+1}(\sigma)+V_k(\mathbf{x}_{k+1}(\sigma))+L\|\mathbf{x}_{k}(\sigma)-\mathbf{x}_{k+1}(\sigma)\|_1+\xi(\sigma)\\&\leq\sum_{l=1}^{k-1}V_l(\mathbf{x}_{k+1}(\sigma))-\sigma^T\mathbf{x}_{k+1}(\sigma)+V_k(\mathbf{x}_{k+1}(\sigma))+b_1NL|\mathbf{x}_{k,i}(\sigma)-\mathbf{x}_{k+1,i}(\sigma)|_1+\xi(\sigma).
\end{split}
\end{equation}
The third inequality follows from the Lipschitz continuity while the last inequality follows from the assumption on $\|\mathbf{x}_k(\sigma)-\mathbf{x}_{k+1}(\sigma)\|_1$. Eq.~\ref{eq_82} provides us with an upper bound of $\sum_{l=1}^{k-1}V_l(\mathbf{x}_k(\sigma))-\sigma^T\mathbf{x}_k(\sigma)+V_k(\mathbf{x}_k(\sigma))$. We next present a lower bound leveraging the approximate optimality.
\begin{equation}\label{eq_83}
\begin{split}
    \sum_{l=1}^{k-1}V_l&(\mathbf{x}_k(\sigma))-\sigma^T\mathbf{x}_k(\sigma)+V_k(\mathbf{x}_k(\sigma))\\&=\sum_{l=1}^{k-1}V_l(\mathbf{x}_k(\sigma))-\hat{\sigma}^T\mathbf{x}_k(\sigma)+V_k(\mathbf{x}_k(\sigma))+b_2LN\mathbf{e}_i^T\mathbf{x}_k(\sigma)\\&\geq \sum_{l=1}^{k-1}V_l(\mathbf{x}_{k+1}(\hat{\sigma}))-\hat{\sigma}^T\mathbf{x}_{k+1}(\hat{\sigma})+V_k(\mathbf{x}_{k+1}(\hat{\sigma}))+b_2LN\mathbf{x}_{k,i}(\sigma)-\xi(\hat{\sigma})\\&=\sum_{l=1}^{k-1}V_l(\mathbf{x}_{k+1}(\hat{\sigma}))-\sigma^T\mathbf{x}_{k+1}(\hat{\sigma})+V_k(\mathbf{x}_{k+1}(\hat{\sigma}))+b_2LN(\mathbf{x}_{k,i}(\sigma)-\mathbf{x}_{k+1,i}(\hat{\sigma}))-\xi(\hat{\sigma})\\&\geq \sum_{l=1}^{k-1}V_l(\mathbf{x}_{k+1}(\sigma))-\sigma^T\mathbf{x}_{k+1}(\sigma)+V_k(\mathbf{x}_{k+1}(\sigma))+b_2LN(\mathbf{x}_{k,i}(\sigma)-\mathbf{x}_{k+1,i}(\hat{\sigma}))-\xi(\hat{\sigma})-\xi(\sigma).
\end{split}    
\end{equation}
The last inequality follows from the approximate optimality of $\mathbf{x}_{k+1}(\sigma)$. Combining Eq.~\ref{eq_82} and Eq.~\ref{eq_83} yields 
\begin{equation}\label{eq_84}
    \mathbf{x}_{k+1,i}(\hat{\sigma})-\mathbf{x}_{k,i}(\sigma)\geq -\frac{1}{b_1}|\mathbf{x}_{k,i}-\mathbf{x}_{k+1,i}|-\frac{3\xi(\sigma)}{b_2NL}-\beta.
\end{equation}
Analogously, we could apply the same derivation to the pair of $\mathbf{x}_{k,i}(\hat{\sigma}), \mathbf{x}_{k+1,i}(\sigma)$ and obtain the following equation
\begin{equation}\label{eq_85}
        \mathbf{x}_{k,i}(\hat{\sigma})-\mathbf{x}_{k+1,i}(\sigma)\geq -\frac{1}{b_1}|\mathbf{x}_{k,i}-\mathbf{x}_{k+1,i}|-\frac{3\xi(\sigma)}{b_2NL}-\beta.
\end{equation}
According to Lemma~\ref{lem_10}, we can obtain
\begin{equation}\label{eq_86}
    \mathbf{x}_{k+1,i}(\hat{\sigma})-\mathbf{x}_{k+1,i}(\sigma)\geq-\frac{2\xi(\sigma)}{b_2NL}-\beta\geq -\frac{3\xi(\sigma)}{b_2NL}-\beta,
\end{equation}
and 
\begin{equation}\label{eq_87}
    \mathbf{x}_{k,i}(\hat{\sigma})-\mathbf{x}_{k,i}(\sigma)\geq-\frac{2\xi(\sigma)}{b_2NL}-\beta\geq -\frac{3\xi(\sigma)}{b_2NL}-\beta.
\end{equation}
The proof is completed by combining Eqs.~\ref{eq_84} - \ref{eq_87}.
\end{proof}
\subsection{Proof for Theorem~\ref{theorem_4}}
\begin{proof}
Noting that $\mathbb{E}[\|\mathbf{x}_{k}-\mathbf{x}_{k+1}\|_1] = \sum_{i=1}^N\mathbb{E}[|\mathbf{x}_{k,i}-\mathbf{x}_{k+1,i}|]$, we proceed to bound $\mathbb{E}[|\mathbf{x}_{k,i}-\mathbf{x}_{k+1,i}|]$ for $i=1,2,...,N$. To investigate the impact of the random perturbation thoroughly, we first define $\mathbb{E}_{-i}[|\mathbf{x}_{k,i}-\mathbf{x}_{k+1,i}|]:=\mathbb{E}\Bigg[|\mathbf{x}_{k,i}-\mathbf{x}_{k+1,i}\;|\; \{\sigma_j\}_{j\neq i}\Bigg]$. Let $\mathbf{x}_{max,i} = \textnormal{max}\{\mathbf{x}_{t,i}(\sigma),\mathbf{x}_{t+1,i}(\sigma)\}$ and $\mathbf{x}_{min,i} = \textnormal{min}\{\mathbf{x}_{t,i}(\sigma),\mathbf{x}_{t+1,i}(\sigma)\}$. Hence, it is immediately obtained that $\mathbb{E}_{-i}[|\mathbf{x}_{k,i}-\mathbf{x}_{k+1,i}|]=\mathbb{E}_{-i}[\mathbf{x}_{max,i}(\sigma)]-\mathbb{E}_{-i}[\mathbf{x}_{min,i}]$. We also define event $\mathcal{Z}$ as\[\mathcal{Z}:=\{\sigma: \|\mathbf{x}_k(\sigma)-\mathbf{x}_{k+1}(\sigma)\|_1\leq b_1N|\mathbf{x}_{k, i}(\sigma)-\mathbf{x}_{k+1,i}(\sigma)|\}.\] The condition in the event is easily satisfied as we can always find a constant $b_1 > 0$. We next give a lower bound for $\mathbf{x}_{min,i}$. Rewriting $\mathbf{x}_{min,i}$ by using the law of total probability, we have
\begin{equation}
\begin{split}
    \mathbb{E}_{-i}[\mathbf{x}_{min,i}(\sigma)] &= \mathbb{P}(\sigma_i < b_2NL)\mathbb{E}_{-i}[\mathbf{x}_{min,i}|\sigma_i< b_2NL]+\mathbb{P}(\sigma_i\geq b_2NL)\mathbb{E}_{-i}[\mathbf{x}_{min,i}|\sigma_i\geq b_2NL]\\&\geq (1-\textnormal{exp}(-b_2NL))(\mathbb{E}_{-i}[\mathbf{x}_{max, i}(\sigma)-\sqrt{N}D])+\textnormal{exp}(-b_2NL)\mathbb{E}_{-i}[\mathbf{x}_{min,i}(\sigma+b_2NL\mathbf{e}_i)].
\end{split}
\end{equation}
The last inequality follows from Lemma~\ref{lem_7} and the fact that the domain of the $i$-th coordinate lies in the interval of length $\sqrt{N}D$ and $\mathbb{E}_{-i}[\mathbf{x}_{min,i}|\sigma_i< b_2NL]$ and $\mathbb{E}_{-i}[\mathbf{x}_{max, i}(\sigma)]$ are points in this interval such that the difference between them should be bounded by $\sqrt{N}D$. This intuitively makes sense also based on Assumption~\ref{assump_2}. Using the event $\mathcal{Z}$ to further give lower bound yields
\begin{equation}
\begin{split}
   \mathbb{E}_{-i}[\mathbf{x}_{min,i}(\sigma)] &\geq (1-\textnormal{exp}(-b_2NL))(\mathbb{E}_{-i}[\mathbf{x}_{max, i}(\sigma)-\sqrt{N}D])\\&+\textnormal{exp}(-b_2NL)\mathbb{P}_{-i}(\mathcal{Z}))\mathbb{E}_{-i}[\mathbf{x}_{min,i}(\sigma+b_2NL\mathbf{e}_i)|\mathcal{Z}]\\&+\textnormal{exp}(-b_2NL)\mathbb{P}_{-i}(\mathcal{Z}^c))\mathbb{E}_{-i}[\mathbf{x}_{min,i}(\sigma+b_2NL\mathbf{e}_i)|\mathcal{Z}^c],
\end{split}
\end{equation}
where $\mathbb{P}_{-i}(\mathcal{Z}):=\mathbb{P}\Bigg(\mathcal{Z}|\{\sigma\}_{j\neq i}\Bigg)$ and $\mathcal{Z}^c$ is the complement of $\mathcal{Z}$. Leveraging Lemmas~\ref{lem_10} and~\ref{lem_11}, $\mathbf{x}_{min,i}$ is further lower bounded as follows.
\begin{equation}
\begin{split}
 \mathbb{E}_{-i}[\mathbf{x}_{min,i}(\sigma)] &\geq  (1-\textnormal{exp}(-b_2NL))(\mathbb{E}_{-i}[\mathbf{x}_{max, i}(\sigma)-\sqrt{N}D])\\&+\textnormal{exp}(-b_2NL)\mathbb{P}_{-i}(\mathcal{Z}))\mathbb{E}_{-i}\Bigg[\mathbf{x}_{max,i}(\sigma)-\frac{1}{b_1}|\mathbf{x}_{k,i}(\sigma)-\mathbf{x}_{k+1,i}(\sigma)|-\frac{3\xi(\sigma)}{b_2NL}-\beta|\mathcal{Z}\Bigg]\\&+\textnormal{exp}(-b_2NL)\mathbb{P}_{-i}(\mathcal{Z}^c))\mathbb{E}_{-i}\Bigg[\mathbf{x}_{min, i}-\frac{2\xi(\sigma)}{b_2NL}-\beta|\mathcal{Z}^c\Bigg]\\&\geq (1-\textnormal{exp}(-b_2NL))(\mathbb{E}_{-i}[\mathbf{x}_{max, i}(\sigma)-\sqrt{N}D])\\&+\textnormal{exp}(-b_2NL)\mathbb{P}_{-i}(\mathcal{Z}))\mathbb{E}_{-i}\Bigg[\mathbf{x}_{max,i}(\sigma)-\frac{1}{b_1}|\mathbf{x}_{k,i}(\sigma)-\mathbf{x}_{k+1,i}(\sigma)|-\frac{3\xi(\sigma)}{b_2NL}-\beta|\mathcal{Z}\Bigg]\\&+\textnormal{exp}(-b_2NL)\mathbb{P}_{-i}(\mathcal{Z}^c))\mathbb{E}_{-i}\Bigg[\mathbf{x}_{max, i}-\frac{1}{b_1N}\|\mathbf{x}_k(\sigma)-\mathbf{x}_{k+1}(\sigma)\|_1-\frac{2\xi(\sigma)}{b_2NL}-\beta|\mathcal{Z}^c\Bigg].
\end{split}
\end{equation}
The first inequality follows from Lemmas~\ref{lem_10} and~\ref{lem_11} while the second one follows directly from the definition of $\mathcal{Z}^c$. Using the law of total probability as well as $\mathbb{P}_{-i}(\mathcal{Z})\leq 1$ and rearranging the right hand side of the last inequality lead to 
\begin{equation}
\begin{split}
\mathbb{E}_{-i}[\mathbf{x}_{min,i}(\sigma)] &\geq  (1-\textnormal{exp}(-b_2NL))(\mathbb{E}_{-i}[\mathbf{x}_{max, i}(\sigma)-\sqrt{N}D])\\&+\textnormal{exp}(-b_2NL)\mathbb{E}_{-i}\Bigg[\mathbf{x}_{max,i}(\sigma)-\frac{3\xi(\sigma)}{b_2NL}-\beta\Bigg]\\&-\textnormal{exp}(-b_2NL)\mathbb{E}_{-i}\Bigg[\frac{1}{b_1}|\mathbf{x}_{k,i}(\sigma)-\mathbf{x}_{k+1,i}(\sigma)|+\frac{1}{b_1N}\|\mathbf{x}_k(\sigma)-\mathbf{x}_{k+1}(\sigma)\|_1\Bigg]\\&\geq\mathbb{E}_{-i}[\mathbf{x}_{max, i}(\sigma)]-b_2N^{\frac{3}{2}}D\eta L-\frac{3\xi(\sigma)}{b_2NL}-\beta\\&-\mathbb{E}_{-i}\Bigg[\frac{1}{b_1}|\mathbf{x}_{k,i}(\sigma)-\mathbf{x}_{k+1,i}(\sigma)|+\frac{1}{b_1N}\|\mathbf{x}_k(\sigma)-\mathbf{x}_{k+1}(\sigma)\|_1\Bigg].
\end{split}
\end{equation}
The last inequality follows from the fact that $\textnormal{exp}(\mathbf{x})\geq 1+\mathbf{x}$. Rearranging the last inequality yields the following relationship
\begin{equation}
\begin{split}
    \mathbb{E}_{-i}[|\mathbf{x}_{k,i}(\sigma)-\mathbf{x}_{k+1,i}(\sigma)|]&\leq\frac{1}{N(b_1-1)}\mathbb{E}_{-i}[\|\mathbf{x}_{k}(\sigma)-\mathbf{x}_{k+1}(\sigma)\|_1]+\frac{b_2-b_1}{b_1-1}N^{\frac{3}{2}}D\eta L\\&+\frac{3b_1\mathbb{E}_{-i}[\xi(\sigma)]}{b_2(b_1-1)NL}+\frac{b_1\beta}{b_1-1}.
\end{split}
\end{equation}
As the last upper bound holds for any $\{\sigma_j\}_{j\neq i}$, we have 
\begin{equation}
\begin{split}
    \mathbb{E}[|\mathbf{x}_{k,i}(\sigma)-\mathbf{x}_{k+1,i}(\sigma)|]&\leq\frac{1}{N(b_1-1)}\mathbb{E}[\|\mathbf{x}_{k}(\sigma)-\mathbf{x}_{k+1}(\sigma)\|_1]+\frac{b_2-b_1}{b_1-1}N^{\frac{3}{2}}D\eta L\\&+\frac{3b_1\mathbb{E}[\xi(\sigma)]}{b_2(b_1-1)NL}+\frac{b_1\beta}{b_1-1}.
\end{split}    
\end{equation}
Recalling $\mathbb{E}[\|\mathbf{x}_k(\sigma)-\mathbf{x}_{k+1}(\sigma)\|_1]=\sum_{i=1}^N\mathbb{E}[|\mathbf{x}_{k,i}-\mathbf{x}_{k+1,i}|]$, we can obtain
\begin{equation}
\mathbb{E}[\|\mathbf{x}_k(\sigma)-\mathbf{x}_{k+1}(\sigma)\|_1]\leq \frac{b_1b_2}{b_1-2}N^{\frac{5}{2}}D\eta L+\frac{3b_1\rho}{b_2L(b_1-2)}+\frac{b_1\beta N}{b_1-2}+\frac{3b_1\beta N}{b_2(b_1-2)L\eta}.
\end{equation}
Substituting the last equation into Eq~\ref{eq_75} results in 
\begin{equation}
\begin{split}
\mathbb{E}\Bigg[\sum_{k=1}^KV_k(\mathbf{x}_k)-\sum_{k=1}^KV_k(\mathbf{x}_*)\Bigg]&\leq \frac{b_1b_2}{b_1-2}N^{\frac{5}{2}}D\eta L^2K+\frac{3b_1\rho K}{b_2(b_1-2)}\\&+\frac{b_1\beta NLK}{b_1-2}+\frac{3b_1\beta NK}{b_2(b_1-2)\eta}+\frac{N(\beta K+\sqrt{N}D)}{\eta}\\&+\rho K.
\end{split}
\end{equation}
Substituting $L\approx\hat{G}\Bigg(1+\frac{\gamma}{1-\lambda}\Bigg)$ into the last inequality leads to 
\begin{equation}
\begin{split}
\mathbb{E}\Bigg[\sum_{k=1}^KV_k(\mathbf{x}_k)-\sum_{k=1}^KV_k(\mathbf{x}_*)\Bigg]&\leq \frac{b_1b_2}{b_1-2}N^{\frac{5}{2}}D\eta K\hat{G}^2\Bigg(1+\frac{\gamma}{1-\lambda}\Bigg)^2+\frac{3b_1\rho K}{b_2(b_1-2)}\\&+\frac{b_1\beta NK\hat{G}}{b_1-2}\Bigg(1+\frac{\gamma}{1-\lambda}\Bigg)+\frac{3b_1\beta NK}{b_2(b_1-2)\eta}+\frac{N(\beta K+\sqrt{N}D)}{\eta}\\&+\rho K,
\end{split}
\end{equation}
which completes the proof.
\end{proof}
\section{OCGD for Strongly Convex and Convex Losses}\label{app_d}
We present the algorithmic framework for the online composite gradient descent (OCGD).

\begin{algorithm}
\caption{Online composite Gradient Descent (OCGD)}
\label{olgd}
\begin{algorithmic}[1]
 \STATE \textbf{Input:} convex sets $\mathcal{X}^N$, $K$, $\mathbf{x}_1=\mathbf{0}\in\mathcal{X}^N$, step size $\{\alpha_k\}$, $\Pi$\\
 \FOR{$k=1\;\textnormal{to}\;K$}
 \FOR{each agent $i$}
 \STATE
  Agent $i$ plays its own strategy $x^i_k$ and observe the losses $V^i_k(x^i_k)$
  \STATE Calculate the feedback $\nabla V^i_k(x^i_k) = \nabla f^i_k(x^i_k) + \frac{1}{\alpha_k}\sum_{j\in Nb(i)}\pi_{ij}\|x^i_k-x^j_k\|$
  \STATE Update:
    $y^i_{k+1} = x^i_k - \alpha_k\nabla V^i_k(x^i_k)$
  \STATE Project:
    $x^i_{k+1} = \mathbf{P}_{\mathcal{X}^i}(y^i_{k+1})$
\ENDFOR
\ENDFOR
\end{algorithmic}
\end{algorithm}
The compact form of OCGD is close to the gradient descent in the offline setting while differing in its online characteristic of regret minimization instead of the static optimization error. We briefly overview Algorithm~\ref{olgd}. For each time step $k$, any agent $i$ plays its own strategy $x^i_k$, and observe its corresponding loss $f^i_k(x^i_k)$, which is the $i$-th component of $\mathcal{F}(\mathbf{x}_k)$, while observing the strategies taken by any other agent $j$ in its neighborhood to calculate the network loss, $c\sum_{j\in Nb(i)}\pi_{ij}\|x^i_k-x^j_k\|^2$, which is the $i$-th component of $c\mathbf{x}_k^T(\mathbf{I}_N-\Pi)\mathbf{x}_k$. It should be noted that the communication way among agents is synchronous. The regularization constant $c$ here is $\frac{1}{2\alpha_k}$. In the algorithmic framework, we combine two losses together to use $V^i_k(x^i_k)$. Then agent $i$ updates itself combining the strategies from all its neighbors and its own first order information $\nabla f^i_k(x^i_k)$. Before repeating this process, a projection step is taken with the Euclidean projection $\mathbf{P}_{\mathcal{X}^i}(y)=\textnormal{arg min}_{x\in\mathcal{X}^i}\|y-x\|$, which is used to guarantee searching the optimal strategy in the feasible space. We next recover the consensus-based online gradient descent using OCGD. We recall $V_k(\mathbf{x}_k)=\mathcal{F}_k(\mathbf{x}_k)+\frac{1}{2\alpha_k}\mathbf{x}_k\|\mathbf{I}_N-\Pi\|^2\mathbf{x}_k$. Based on the expression of $V_k(\mathbf{x}_k)$, calculating its gradient yields $\nabla V_k(\mathbf{x}_k) = \nabla \mathcal{F}_k(\mathbf{x}_k) + \frac{1}{\alpha_k}(\mathbf{I}_N-\Pi)\mathbf{x}_k$. Substituting $\nabla V(\mathbf{x}_k)$ into $\mathbf{y}_{k+1}=\mathbf{x}_k-\alpha_k\nabla V_k(\mathbf{x}_k)$ results in $\mathbf{y}_{k+1}=\Pi\mathbf{x}_k-\alpha\nabla \mathcal{F}_k(\mathbf{x}_k)$, which is the consensus-based online gradient descent. Hence using composite regret leads to a simpler form, based on which we can follow the analysis techniques for the centralized online gradient descent. However, due to the consensus term in $\nabla V(\mathbf{x}_k)$, we will have to derive a new upper bound for it instead of directly setting the assumption. 
For the main result, we apply OCGD to solve distributed online optimization problems when the loss functions are strongly convex and non-strongly convex.
By revisiting the definition of strong convexity, the following relationship can be obtained.
\begin{defn}\label{strong_convexity}
A loss function $f^i$ is called $\mu_i$-strongly convex for $x,y\in\mathcal{X}^i$ if the following inequality is satisfied
\begin{equation}
    f^i(y)\geq f^i(x)+\nabla f^i(x)^T(y-x)+\frac{\mu_i}{2}\|y-x\|^2
\end{equation}
\end{defn}
It is immediately obtained that $V(\mathbf{x})$ is strongly convex with $\mu:=\mu_m+\frac{1}{2\alpha}(1-\lambda)$, where $\mu_m=\textnormal{min}\{\mu_1,\mu_2,...,\mu_N\}$. With this in hand, we next present the main result for the strongly convex loss function. In this context, we require the loss of each agent to be strongly convex, which could be a quite strong condition. However, our results could be relaxed to a weaker scenario where $\mathcal{F}(\mathbf{x})$ is strongly convex, but not all $f^i(x^i)$ need to be strongly convex.   
\begin{thm}\label{theorem_1}
Let all assumptions hold. OCGD with $\textnormal{SC-regret}_K$ and the step size $\alpha_k=\frac{1}{\mu k}, \frac{1}{\alpha_0}= 0, k=0,1,...,K$ guarantees the logarithmic regret $\mathcal{O}(\textnormal{log}K)$ for all $K\geq 1$ when the loss function $\mathcal{F}(\cdot)$ is strongly convex.
\end{thm}
For non-strongly convex problems, by varying the step size, one can achieve the following result.
\begin{thm}\label{theorem_2}
Let all assumptions hold. OCGD 1) with $\textnormal{SC-regret}_K$ and the step size $\alpha=\frac{\sqrt{N}D}{C\sqrt{K}}$ guarantees the sublinear regret $\mathcal{O}(\sqrt{K})$, 2) with $\textnormal{DC-regret}_K$ and the step size $\alpha=\frac{\sqrt{\sqrt{N}D(\sqrt{N}D+3P_K)}}{C\sqrt{K}}$ guarantees the subliear regret $\mathcal{O}(\sqrt{K+KP_K})$, for all $K\geq 1$ when the loss functions are non-strongly convex, where $C=G\Bigg(1+\frac{2}{1-\lambda}\Bigg)$.
\end{thm}
We will show that for either the static or dynamic composite regret, Theorem~\ref{theorem_2} can always hold, but the corresponding results differ in some constants due to the path variations of the optimizer $\mathbf{x}_{*,k}$ when taking the $\textnormal{DC-regret}_K$.
\section{Proof of Theorem~\ref{theorem_1} and Theorem~\ref{theorem_2}}\label{app_e}
Before proving the theorem, we first present a few auxiliary technical lemmas.
We provide an auxiliary lemma to upper bound the first order information of network loss, $\|x^i_k - x^j_k\|$, which assists in bounding $\nabla V(\mathbf{x}_k)$. Instead of directly bounding $\|x^i_k - x^j_k\|$, we use a reference sequence $\{\hat{x}_k\}$, where $\hat{x}_k = \frac{1}{N}\sum_{i=1}^N x^i_k$ is the ensemble average for all agents. Define $\hat{\mathbf{x}}_k=[\hat{x}_k;\hat{x}_k;...;\hat{x}_k]_N$ such that $\hat{\mathbf{x}}_k=\frac{1}{N}\mathbf{1}\mathbf{1}^T\mathbf{x}_k$. By induction, we can get
\begin{equation}
    \mathbf{x}_k=\Pi^{k-1}\mathbf{x}_1 -\sum_{s=1}^{k-1}\Pi^{k-1-s}\alpha_s\nabla \mathcal{F}_s(\mathbf{x}_s) 
\end{equation}
As $\|x^i_k-\hat{x}_k\|\leq\|\mathbf{x}_k-\hat{\mathbf{x}}_k\|$, instead of proving the upper bound of $\|x^i_k-\hat{x}_k\|$, we can show that $\|\mathbf{x}_k-\hat{\mathbf{x}}_k\|$ is bounded correspondingly.
\begin{lem}\label{consensus_lem}
Let Assumption~\ref{assump_1} hold. The sequence $\{x^i_k\}$ generated by OCGD satisfies the following relationship
\begin{equation}
    \|x^i_k-\hat{x}_k\|\leq G\sum_{s=1}^{k-1}\alpha_s\lambda^{k-1-s}.
\end{equation}
\end{lem}
The proof of Lemma~\ref{consensus_lem} follows similarly from the Lemma V.2. in~\citep{berahas2018balancing} where they set a constant step size. This result has been well studied in distributed offline optimization. Even in an online setting, this result still holds, as the underlying graph is not subject to changes, plus that the losses are still Lipschitz continuous. Similar results can be seen from~\citep{li2018distributed} and~\citep{sharma2020distributed}. Adopting Lemma~\ref{consensus_lem}, we next show the upper bound of $\nabla V_k(\mathbf{x}_k)$, which will be useful for proving the main results.
\begin{lem}\label{lem_2}
Let Assumption~\ref{assump_1} hold. The sequence $\{x^i_k\}$ generated by OCGD satisfies the following relationship in the vector form
\begin{equation}
    \|\nabla V_k(\mathbf{x}_k)\|\leq G\Bigg(1+\frac{2}{\alpha_k}\sum_{s=1}^{k-1}\alpha_s\lambda^{k-1-s}\Bigg),
\end{equation}
\end{lem}
\textbf{Implication from Lemma~\ref{lem_2}}. Lemma~\ref{lem_2} suggests the explicit relationship between the gradient of $V_k(\mathbf{x}_k)$ and the network topology, which facilitates the analysis of the optimality. When the step size is constant, i.e., $\alpha_k=\alpha, k=1,2,...,K$, one can easily get the upper bound using the geometric series. Also, the parameter $c$ in the composite regret under this scenario is set $\frac{1}{2\alpha}$. Unfortunately, when the step size is time-varying such that $c$ is also time-varying, the composite regret may not necessarily hold at any time step $k$, particularly when $k$ is sufficiently large, which could turn $\alpha_k$ into a sufficiently small value. Thus, we need to investigate the part of $\frac{1}{\alpha_k}(\mathbf{I}_N-\Pi)\mathbf{x}_k$ in $\nabla V_k(\mathbf{x}_k)$ as $k\to\infty$.
For the constant step size, we provide a lemma to bound the $\nabla V_k(\mathbf{x}_k)$. 
\begin{corol}\label{coro_1}
Let Assumption~\ref{assump_1} hold. The sequence $\{x^i_k\}$ generated by OCGD satisfies the following relationship in the vector form
\begin{equation}
    \|\nabla V_k(\mathbf{x}_k)\|\leq C,
\end{equation}
where $C=G(1+\frac{2}{1-\lambda})$.
\end{corol}
The proof for Corollary~\ref{coro_1} is obtained by setting $\alpha_k=\alpha, k=1,2,...,K$ and using $\sum_{k=1}^{\infty}\lambda^k=\frac{1}{1-\lambda}$, with Lemma~\ref{lem_2}. For the diminishing step size, we study two kinds of step sizes that are $\alpha_k=\frac{B}{k}$ and $\alpha_k=\frac{B}{\sqrt{k}}, B> 0$, for all $k\geq 1$. Note that $B$ will be specified explicitly in the main analysis and for the generic purposes, we temporarily use $B$ to represent the constant in the step size. Before giving the upper bound for $\nabla V_k(\mathbf{x}_k)$, one auxiliary lemma is presented as follows.
\begin{lem}\label{lem_3}
Let Assumption~\ref{assump_1} hold. The sequence $\{x^i_k\}$ generated by OCGD in the vector form has
\begin{equation}
    \alpha_k\|(\mathbf{I}_N-\Pi)\mathbf{x}_k\|\leq 2G\sum_{s=1}^{k-1}\alpha^2_s\lambda^{k-1-s}.
\end{equation}
\end{lem}
Lemma~\ref{lem_3} suggests that for any $k$, $\alpha_k\|(\mathbf{I}_N-\Pi)\mathbf{x}_k\|$ is upper bounded such that \[\underset{k}{\textnormal{sup}}\{\alpha_k\|(\mathbf{I}_N-\Pi)\mathbf{x}_k\|\}\leq2G\sum_{s=1}^{k-1}\alpha^2_s\lambda^{k-1-s},\]which suggests that
\[\underset{k}{\textnormal{sup}}\Bigg\{\alpha_k^2\frac{\|(\mathbf{I}_N-\Pi)\mathbf{x}_k\|}{\alpha_k}\Bigg\}\leq2G\sum_{s=1}^{k-1}\alpha^2_s\lambda^{k-1-s}.\]
By summing up from $k=2$ to $\infty$, the following inequality is obtained
\[\sum_{k=2}^{\infty}\underset{k}{\textnormal{sup}}\Bigg\{\alpha_k^2\frac{\|(\mathbf{I}_N-\Pi)\mathbf{x}_k\|}{\alpha_k}\Bigg\}\leq2G\sum_{k=2}^{\infty}\sum_{s=1}^{k-1}\alpha^2_s\lambda^{k-1-s}.\]
The left hand side of the last inequality satisfies \[\sum_{k=2}^{\infty}\alpha_k^2\underset{k}{\textnormal{sup}}\Bigg\{\frac{\|(\mathbf{I}_N-\Pi)\mathbf{x}_k\|}{\alpha_k}\Bigg\}\leq\sum_{k=2}^{\infty}\underset{k}{\textnormal{sup}}\Bigg\{\alpha_k^2\frac{\|(\mathbf{I}_N-\Pi)\mathbf{x}_k\|}{\alpha_k}\Bigg\}.\]
Then the following lemma summarizes the upper bound of $\frac{1}{\alpha_k}\|(\mathbf{I}_N-\Pi)\mathbf{x}_k\|$ when the step size is diminishing.
\begin{lem}\label{lem_4}
Let Assumption~\ref{assump_1} hold. If the step size satisfies $\alpha_k=\frac{B}{k}, B> 0, \forall k\geq 2$, we have
\begin{equation}\label{dimini_1}
    \frac{1}{\alpha_k}\|(\mathbf{I}_N-\Pi)\mathbf{x}_k\|\leq \frac{2G}{1-\lambda};
\end{equation}
if the step size satisfies $\alpha_k=\frac{B}{\sqrt{k}}, B> 0, \forall k\geq 2$, we have
\begin{equation}\label{dimini_2}
    \frac{1}{\alpha_k}\|(\mathbf{I}_N-\Pi)\mathbf{x}_k\|\leq \frac{3.42G}{1-\lambda}.
\end{equation}
\end{lem}
Although the constant $B$ in this context is not explicitly displaying in the bound due to the cancellation in the proof, as mentioned above, in the subsequent analysis, $B$ is explicitly specified based on diverse scenarios. With Lemma~\ref{lem_4}, the upper bound of $\nabla V_k(\mathbf{x}_k)$ under different diminishing step sizes are attained as follows. When $\alpha_k=\frac{B}{k}$
\begin{equation}
    \|\nabla V_k(\mathbf{x}_k)\|\leq G\Bigg(1+\frac{2}{1-\lambda}\Bigg),
\end{equation}
which is surprisingly the same as that with the constant step size. While for $\alpha_k=\frac{B}{\sqrt{k}}$, we have
\begin{equation}
    \|\nabla V_k(\mathbf{x}_k)\|\leq G\Bigg(1+\frac{3.42}{1-\lambda}\Bigg),
\end{equation}
which is slightly larger. Intuitively, this makes sense because among three different variants of step sizes, $\alpha_k=\frac{B}{\sqrt{k}}$ is the slowest one that could result in a larger bound, as the upper bound applies to all $k\geq2$, not only when $k$ is sufficiently large.
With Lemma~\ref{lem_4} in hand, we now proceed to prove Theorem~\ref{theorem_1}, which is provided in Appendix A.
When substituting $C=G\Bigg(1+\frac{2}{1-\lambda}\Bigg)$ into the Eq.~\ref{eq_33}, we have
\begin{equation}
    \textnormal{SC-regret}_K\leq\frac{G^2}{2\mu}\Bigg(1+\frac{2}{1-\lambda}\Bigg)^2(1+\textnormal{log}K).
\end{equation}
It can be observed from the last inequality that adopting the static form of the composite regret enables the OCGD to obtain a logarithmic regret bound when the losses are strongly convex, which matches the regret bound in the centralized scenario. However, a significant difference is that the regret bound is dependent on the network topology as it is inversely proportional to the square of the spectral gap $1-\lambda$. The similar relationship between the function error and the spectral gap has been widely shown in the distributed offline optimization. We also notice that this should be expected in the online setting where the underlying graph has no changes compared to that in the offline setting.

One observation is that with the step size in the order of $\mathcal{O}(\frac{1}{\sqrt{K}})$, the OCGD enables the $\textnormal{SC-regret}_K$ to grow with a sublinear order of $\mathcal{O}(\sqrt{K})$ and the $\textnormal{DC-regret}_K$ to grow with a sublinear of $\mathcal{O}(\sqrt{K+KP_K})$, which matches the state-of-the-art performance~\citep{bastianello2020distributed,sharma2020distributed, yi2020distributed}. This suggests a good performance for the OCGD as in~\citep{li2018distributed}. Compared to the primal-dual approaches, the proposed OCGD maintains the simple algorithmic update as the centralized OGD has. Another implication that can be made from the results of Theorem~\ref{theorem_2} is that using the $\textnormal{DC-regret}_K$ results in a larger constant in the regret bound due to the path variation. Intuitively this makes sense because the running best policy generates variances during the period. If for all $k$ it holds that $\mathbf{x}_{*,k+1}=\mathbf{x}_{*,k}$, then $P_K=0$ such that the regret bound for the $\textnormal{DC-regret}_K$ is exactly the same as that of the $\textnormal{SC-regret}_K$. This suggests that all running best policies are the same fixed best policy.
\begin{rem}
Theorem~\ref{theorem_2} suggests that the regret bound has an explicit relationship with the network connectivity. When substituting $C=G(1+\frac{2}{1-\lambda})$ into the regret bounds, it clearly shows that the regret bound is inversely proportional to the spectral gap $1-\lambda$, which varies depending on the specific network topology. When the network is sparse, such a ring network, $1-\lambda$ is less than that of a fully connected network. Thus, compared to a fully connected network, within a ring network, the regret bound is relatively larger. This also validates the proposal of the composite regret involving an extra network loss term that plays a critical role to take into account the ``spatial variations" among diverse agents, which presents an explicit relationship between $P_K$ and $1-\lambda$.  
\end{rem}
As Table~\ref{table:findings} in Section~\ref{table_section} shows, when the step size $\alpha_k=\frac{B}{\sqrt{k}}, \forall k\geq 1$, using the $\textnormal{SC-regret}_K$ can achieve the same order of regret bound and we provide all the proof in the sequel.
\subsection{Proof for Lemma~\ref{lem_2}}
\begin{proof}
Based on the expression of $\nabla V_k(\mathbf{x}_k)$, one can obtain\[\nabla V_k(\mathbf{x}_k)=\nabla\mathcal{F}_k(\mathbf{x}_k)+\frac{1}{\alpha_k}(\mathbf{I}_N-\Pi)\mathbf{x}_k.\]
It is immediately obtained that \[\|\nabla V_k(\mathbf{x}_k)\|\leq\|\nabla\mathcal{F}_k(\mathbf{x}_k)\|+\frac{1}{\alpha_k}\|\mathbf{x}_k-\Pi\mathbf{x}_k\|.\] For the second term on the right hand side of the last inequality, we can rewrite it as \[\frac{1}{\alpha_k}\|\mathbf{x}_k-\hat{\mathbf{x}}_k+\hat{\mathbf{x}}_k-\Pi\mathbf{x}_k\|\leq\frac{1}{\alpha_k}(\|\mathbf{x}_k-\hat{\mathbf{x}}_k\|+\|\hat{\mathbf{x}}_k-\Pi\mathbf{x}_k\|)\]
As $\frac{1}{N}\mathbf{1}\mathbf{1}^T\Pi = \frac{1}{N}\Pi\mathbf{1}\mathbf{1}^T=\frac{1}{N}\mathbf{1}\mathbf{1}^T$, we can obtain \[\|\nabla V_k(\mathbf{x}_k)\|\leq\|\nabla\mathcal{F}_k(\mathbf{x}_k)\|+\frac{2}{\alpha_k}\|\mathbf{x}_k-\hat{\mathbf{x}}_k\|. \]With Assumption~\ref{assump_1} and Lemma~\ref{consensus_lem}, we can obtain\[\|\nabla V_k(\mathbf{x}_k)\|\leq G+\frac{2G}{\alpha_k}\sum_{s=1}^{k-1}\alpha_s\lambda^{k-1-s},\]which completes the proof.
\end{proof}
\subsection{Proof for Lemma~\ref{lem_3}}
\begin{proof}
According to Lemma~\ref{lem_2}, we attain 
\[\alpha_k\|(\mathbf{I}_N-\Pi)\mathbf{x}_k\|\leq 2G\alpha_k\sum_{s=1}^{k-1}\alpha_s\lambda^{k-1-s}.\]
As the diminishing step size is non-increasing, then we have
\[\alpha_k\sum_{s=1}^{k-1}\alpha_s\lambda^{k-1-s}\leq \sum_{s=1}^{k-1}\alpha_s^2\lambda^{k-1-s},\]which completes the proof.
\end{proof}
\subsection{Proof for Lemma~\ref{lem_4}}
\begin{proof}
Recall
\begin{equation}\label{eq_23}
    \sum_{k=2}^{\infty}\alpha_k^2\underset{k}{\textnormal{sup}}\Bigg\{\frac{\|(\mathbf{I}_N-\Pi)\mathbf{x}_k\|}{\alpha_k}\Bigg\}\leq 2G\sum^{\infty}_{k=2}\sum_{s=1}^{k-1}\alpha_s^2\lambda^{k-1-s}.
\end{equation}
Hence, we have
\begin{equation}\label{eq_24}
    \underset{k}{\textnormal{sup}}\Bigg\{\frac{\|(\mathbf{I}_N-\Pi)\mathbf{x}_k\|}{\alpha_k}\Bigg\}\leq\frac{2G\sum^{\infty}_{k=2}\sum_{s=1}^{k-1}\alpha_s^2\lambda^{k-1-s}}{\sum_{k=2}^{\infty}\alpha_k^2}.
\end{equation}
When $\alpha_k$ is set $\frac{B}{k}$ such that $\sum^{\infty}_{k=2}\sum_{s=1}^{k-1}\alpha_s^2\lambda^{k-1-s}=\sum_{s=2}^{\infty}\frac{B^2}{s^2}\sum_{k=s+1}^{\infty}\lambda^{k-1-s}=\sum_{s=2}^{\infty}\frac{B^2}{s^2(1-\lambda)}$, then combining Eq.~\ref{eq_24}, we can obtain Eq.~\ref{dimini_1}. When $\alpha_k=\frac{B}{\sqrt{k}}$, we need to provide the upper and lower bounds for the term $\sum_{k=2}^K\alpha_k^2$. Due to properties of the Harmonic series~\citep{nedic2014distributed}, we have $\textnormal{ln}(K+1)-\textnormal{ln}2=\int^{K+1}_{2}\frac{1}{k}<\sum_{k=2}^K\alpha_k^2\leq \textnormal{ln}K$. Thus, the following relationship is attained
\begin{equation}
    \underset{k}{\textnormal{sup}}\Bigg\{\frac{\|(\mathbf{I}_N-\Pi)\mathbf{x}_k\|}{\alpha_k}\Bigg\}\leq\lim_{K\to\infty}\frac{2G\sum^{K}_{k=2}\sum_{s=1}^{k-1}\alpha_s^2\lambda^{k-1-s}}{\sum_{k=2}^{K}\alpha_k^2}\leq \lim_{K\to\infty}\frac{2G\textnormal{ln}K}{(1-\lambda)(\textnormal{ln}(K+1)-\textnormal{ln}2)}
\end{equation}
As $\frac{\textnormal{ln}K}{\textnormal{ln}(K+1)-\textnormal{ln}2}, \forall K\geq 2$ is bounded above by $\frac{\textnormal{ln}2}{\textnormal{ln}3-\textnormal{ln}2}\approx 1.71$, when $\alpha_k=\frac{B}{\sqrt{k}}$, Eq.~\ref{dimini_2} is resultant.
\end{proof}
\subsection{Proof of Theorem~\ref{theorem_1}}
\begin{proof}
Let $\mathbf{x}_*\in\textnormal{arg min}_{\mathbf{x}\in\mathcal{X}^N}\sum_{k=1}^KV_k(\mathbf{x})$. Recall the definition of the $\textnormal{SC-regret}_K$
\[\textnormal{SC-regret}_K=\sum_{k=1}^KV_k(\mathbf{x}_k)-\sum_{k=1}^KV_k(\mathbf{x}_*).\]
Applying the definition of $\mu$-strong convexity to two points $\mathbf{x}_k$ and $\mathbf{x}_*$ at $V_k(\mathbf{x}_k)$, we have
\begin{equation}\label{eq_28}
    2(V_k(\mathbf{x}_k)-V_k(\mathbf{x}_*))\leq2\nabla V_k(\mathbf{x}_k)^T(\mathbf{x}_k-\mathbf{x}_*)-\mu\|\mathbf{x}_k-\mathbf{x}_*\|^2.
\end{equation}
We next upper bound $\nabla V_k(\mathbf{x}_k)^T(\mathbf{x}_k-\mathbf{x}_*)$. Recalling the update rule of the OCGD and applying the nonexpansiveness property of the projection lead to
\begin{equation}
    \|\mathbf{x}_{k+1}-\mathbf{x}_*\|^2=\|\mathbf{P}_{\mathcal{X}^N}(\mathbf{x}_k-\alpha_k\nabla V_k(\mathbf{x}_k))-\mathbf{x}_*\|^2\leq\|\mathbf{x}_k-\alpha_k\nabla V_k(\mathbf{x}_k)-\mathbf{x}_*\|^2.
\end{equation}
Expanding the last inequality yields
\begin{equation}
    \|\mathbf{x}_{k+1}-\mathbf{x}_*\|^2\leq\|\mathbf{x}_{k}-\mathbf{x}_*\|^2+\alpha_k^2\|\nabla V_k(\mathbf{x}_k)\|^2-2\alpha_k\nabla V_k(\mathbf{x}_k)^T(\mathbf{x}_{k}-\mathbf{x}_*),
\end{equation}
and
\begin{equation}
    2\nabla V_k(\mathbf{x}_k)^T(\mathbf{x}_{k}-\mathbf{x}_*)\leq\frac{\|\mathbf{x}_{k}-\mathbf{x}_*\|^2-\|\mathbf{x}_{k+1}-\mathbf{x}_*\|^2}{\alpha_k}+\alpha_kC.
\end{equation}
Summing from $k=1$ to $K$, setting $\alpha_k=\frac{1}{\mu k}$ (defining $B=\frac{1}{\mu}$ and $\frac{1}{\alpha_0}=0$), and combining with Eq.~\ref{eq_28}, we have
\begin{equation}
    2\sum_{k=1}^K(V_k(\mathbf{x}_k)-V_k(\mathbf{x}_*))\leq\sum_{k=1}^K\|\mathbf{x}_k-\mathbf{x}_*\|^2\Bigg(\frac{1}{\alpha_k}-\frac{1}{\alpha_{k-1}}-\mu\Bigg)+C^2\sum_{k=1}^K\alpha_k.
\end{equation}
As $\frac{1}{\alpha_0}=0$ and $\|\mathbf{x}_{K+1}-\mathbf{x}_*\|^2\geq 0$, we attain
\begin{equation}\label{eq_33}
    2\sum_{k=1}^K(V_k(\mathbf{x}_k)-V_k(\mathbf{x}_*))\leq 0+C^2\sum_{k=1}^K\frac{1}{\mu k}\leq \frac{C^2}{\mu}(1+\textnormal{log}K),
\end{equation}
which completes the proof.
\end{proof}

\subsection{Proof of Theorem~\ref{theorem_2}}
\subsubsection{Proof for Theorem~\ref{theorem_2} with Constant Step Size}
\begin{proof}
As the step size in this theorem is constant, we let $\alpha_k=\alpha, \forall k$.

\textbf{SC-regret:}
Let $\mathbf{x}_*\in\textnormal{argmin}_{\mathbf{x}\in\mathcal{X}^N}\sum_{k=1}^KV_k(\mathbf{x}_k)$. Recall the definition of $\textnormal{SC-regret}_K$\[\textnormal{SC-regret}_K=\sum_{k=1}^KV_k(\mathbf{x}_k)-\sum_{k=1}^KV_k(\mathbf{x}_*).\] Applying the definition of convexity to $\mathbf{x}_k$ and $\mathbf{x}_*$ yields
\begin{equation}
    V_k(\mathbf{x}_k)-V_k(\mathbf{x}_*)\leq \nabla V_k(\mathbf{x}_k)^T(\mathbf{x}_k-\mathbf{x}_*).
\end{equation}
We proceed to get the upper bound for $\nabla V_k(\mathbf{x}_k)^T(\mathbf{x}_k-\mathbf{x}_*)$. Using the update law of OCGD and the nonexpansiveness property of the projection, we get 
\begin{equation}
    \|\mathbf{x}_{k+1}-\mathbf{x}_*\|^2=\|\mathbf{P}_{\mathcal{X}^N}(\mathbf{x}_k-\alpha\nabla V_k(\mathbf{x}_k))-\mathbf{x}_*\|^2\leq\|\mathbf{x}_k-\alpha\nabla V_k(\mathbf{x}_k)-\mathbf{x}_*\|^2
\end{equation}
Hence, 
\begin{equation}
    \|\mathbf{x}_{k+1}-\mathbf{x}_*\|^2\leq \|\mathbf{x}_k-\mathbf{x}_*\|^2 + \alpha^2\|\nabla V_k(\mathbf{x}_k)\|^2 - 2\alpha\nabla V_k(\mathbf{x}_k)^T(\mathbf{x}_k-\mathbf{x}_*)
\end{equation}
and 
\begin{equation}
    2\nabla V_k(\mathbf{x}_k)^T(\mathbf{x}_k-\mathbf{x}_*)\leq\frac{1}{\alpha}(\|\mathbf{x}_k-\mathbf{x}_*\|^2-\|\mathbf{x}_{k+1}-\mathbf{x}_*\|^2) + \alpha C^2.
\end{equation}
The last equation follows from the result of Lemma~\ref{coro_1}. By summing up the last inequality from $k=1$ to $K$, we obtain
\begin{equation}\label{eq_21}
    2\sum_{k=1}^K\nabla V_k(\mathbf{x}_k)^T(\mathbf{x}_k-\mathbf{x}_*)\leq \frac{1}{\alpha}\sum_{k=1}^K(\|\mathbf{x}_k-\mathbf{x}_*\|^2-\|\mathbf{x}_{k+1}-\mathbf{x}_*\|^2)+K\alpha C^2.
\end{equation}
The right hand side of the Eq.~\ref{eq_21} can be bounded above by \[\frac{1}{\alpha}\|\mathbf{x}_1-\mathbf{x}_*\|^2+K\alpha C^2.\]
Hence, combining the convexity, we get
\begin{equation}
    2\sum_{k=1}^K(V_k(\mathbf{x}_k)-V_k(\mathbf{x}_*))\leq\frac{1}{\alpha}\|\mathbf{x}_1-\mathbf{x}_*\|^2+K\alpha C^2
\end{equation}
Using Assumption~\ref{assump_2}, the following is obtained
\begin{equation}
    2\sum_{k=1}^K(V_k(\mathbf{x}_k)-V_k(\mathbf{x}_*))\leq\frac{ND^2}{\alpha}+K\alpha C^2
\end{equation}
Defining the step size as
\[\alpha=\frac{\sqrt{N}D}{C\sqrt{K}},\] we have
\begin{equation}
    \sum_{k=1}^K(V_k(\mathbf{x}_k)-V_k(\mathbf{x}_*))\leq DC\sqrt{NK}.
\end{equation}
\textbf{DC-regret:} For the dynamic regret, we investigate $\|\mathbf{x}_{k+1}-\mathbf{x}_{*,k+1}\|^2$. Define $z_{k+1}=\|\mathbf{x}_{k+1}-\mathbf{x}_{*,k+1}\|^2$. Then we have
\begin{equation}
    z_{k+1}=\|\mathbf{x}_{k+1}-\mathbf{x}_{*,k}+\mathbf{x}_{*,k}-\mathbf{x}_{*,k+1}\|^2.
\end{equation}
Expanding the last equality yields
\begin{equation}
   z_{k+1}=\|\mathbf{x}_{k+1}-\mathbf{x}_{*,k}\|^2+\|\mathbf{x}_{*,k}-\mathbf{x}_{*,k+1}\|^2+2(\mathbf{x}_{k+1}-\mathbf{x}_{*,k})^T(\mathbf{x}_{*,k}-\mathbf{x}_{*,k+1}).
\end{equation}
Substituting the update law of OCGD and using the nonexpansiveness property of the projection leads to
\begin{equation}
   z_{k+1}=\|\mathbf{x}_k-\alpha\nabla V_k(\mathbf{x}_k)-\mathbf{x}_{*,k}\|^2+\|\mathbf{x}_{*,k}-\mathbf{x}_{*,k+1}\|^2+2(\mathbf{x}_{k+1}-\mathbf{x}_{*,k})^T(\mathbf{x}_{*,k}-\mathbf{x}_{*,k+1}),
\end{equation}
which suggests
\begin{equation}
    z_{k+1}\leq\|\mathbf{x}_k-\alpha\nabla V_k(\mathbf{x}_k)-\mathbf{x}_{*,k}\|^2 + \|\mathbf{x}_{*,k}-\mathbf{x}_{*,k+1}\|^2 + 2\|\mathbf{x}_{k+1}-\mathbf{x}_{*,k}\|\|\mathbf{x}_{*,k}-\mathbf{x}_{*,k+1}\|.
\end{equation}
The last inequality holds due to the Cauthy-Schwarz inequality. We then obtain
\begin{equation}
    z_{k+1}\leq z_k + \alpha^2\|\nabla V(\mathbf{x}_k)\|^2 - 2\alpha \nabla V(\mathbf{x}_k)^T(\mathbf{x}_k-\mathbf{x}_{*,k})+3\sqrt{N}D\|\mathbf{x}_{*,k}-\mathbf{x}_{*,k+1}\|
\end{equation}
We proceed to bound $\nabla V(\mathbf{x}_k)^T(\mathbf{x}_k-\mathbf{x}_{*,k})$. By modifying the last inequality, the following inequality is acquired
\begin{equation}\label{eq_30}
    2\nabla V(\mathbf{x}_k)^T(\mathbf{x}_k-\mathbf{x}_{*,k})\leq\frac{1}{\alpha}(z_k-z_{k+1})+\alpha\|\nabla V(\mathbf{x}_k)\|^2+\frac{3\sqrt{N}D}{\alpha}\|\mathbf{x}_{*,k}-\mathbf{x}_{*,k+1}\|.
\end{equation}
By summing up Eq.~\ref{eq_30} from 1 to $K$, one can
\begin{equation}\label{eq_31}
    2\sum_{k=1}^K\nabla V(\mathbf{x}_k)^T(\mathbf{x}_k-\mathbf{x}_{*,k})\leq\frac{1}{\alpha}\sum_{k=1}^K(z_k-z_{k+1})+K\alpha C^2+\frac{3\sqrt{N}D}{\alpha}\sum_{k=1}^K\|\mathbf{x}_{*,k}-\mathbf{x}_{*,k+1}\|.
\end{equation}
Eq.~\ref{eq_31} follows from Lemma~\ref{coro_1}. The first term on the right hand side of Eq.~\ref{eq_31} can be bounded by $\frac{ND^2}{\alpha}$. While the third term of Eq.~\ref{eq_31} is bounded by $\frac{3\sqrt{N}DP_K}{\alpha}$. Hence, we arrive at
\begin{equation}
    2\sum_{k=1}^K\nabla V(\mathbf{x}_k)^T(\mathbf{x}_k-\mathbf{x}_{*,k})\leq\frac{\sqrt{N}D(\sqrt{N}D+3P_K)}{\alpha} + K\alpha C^2.
\end{equation}
With the convexity, we have
\begin{equation}
    2\sum_{k=1}^K(V_k(\mathbf{x}_k)-V_k(\mathbf{x}_{*,k}))\leq \frac{\sqrt{N}D(\sqrt{N}D+3P_K)}{\alpha} + K\alpha C^2.
\end{equation}
Define the step size as \[\alpha=\frac{\sqrt{\sqrt{N}D(\sqrt{N}D+3P_K)}}{C\sqrt{K}}.\]
Therefore, we can obtain
\begin{equation}
    \sum_{k=1}^K(V_k(\mathbf{x}_k)-V_k(\mathbf{x}_{*,k}))\leq\sqrt{\sqrt{N}D(\sqrt{N}D+3P_K)}C\sqrt{K},
\end{equation}
which completes the proof.
\end{proof}
\subsubsection{Proof for Theorem~\ref{theorem_2} with Diminishing Step Size}
\begin{proof}
\textbf{SC-regret:}
Let $\mathbf{x}_*\in\textnormal{argmin}_{\mathbf{x}\in\mathcal{X}}\sum_{k=1}^KV_k(\mathbf{x}_k)$. Recall the definition of $\textnormal{SC-regret}_K$\[\textnormal{SC-regret}_K=\sum_{k=1}^KV_k(\mathbf{x}_k)-\sum_{k=1}^KV_k(\mathbf{x}_*).\] Applying the definition of convexity to $\mathbf{x}_k$ and $\mathbf{x}_*$ yields
\begin{equation}
    V_k(\mathbf{x}_k)-V_k(\mathbf{x}_*)\leq \nabla V_k(\mathbf{x}_k)^T(\mathbf{x}_k-\mathbf{x}_*).
\end{equation}
We proceed to get the upper bound for $\nabla V_k(\mathbf{x}_k)^T(\mathbf{x}_k-\mathbf{x}_*)$. Using the update law of OCGD and the nonexpansiveness property of the projection, we get 
\begin{equation}
    \|\mathbf{x}_{k+1}-\mathbf{x}_*\|^2=\|\mathbf{P}_{\mathcal{X}^N}(\mathbf{x}_k-\alpha_k\nabla V_k(\mathbf{x}_k))-\mathbf{x}_*\|^2\leq\|\mathbf{x}_k-\alpha_k\nabla V_k(\mathbf{x}_k)-\mathbf{x}_*\|^2
\end{equation}
Hence, 
\begin{equation}
    \|\mathbf{x}_{k+1}-\mathbf{x}_*\|^2\leq \|\mathbf{x}_k-\mathbf{x}_*\|^2 + \alpha^2_k\|\nabla V_k(\mathbf{x}_k)\|^2 - 2\alpha_k\nabla V_k(\mathbf{x}_k)^T(\mathbf{x}_k-\mathbf{x}_*)
\end{equation}
and 
\begin{equation}
    2\nabla V_k(\mathbf{x}_k)^T(\mathbf{x}_k-\mathbf{x}_*)\leq\frac{1}{\alpha_k}(\|\mathbf{x}_k-\mathbf{x}_*\|^2-\|\mathbf{x}_{k+1}-\mathbf{x}_*\|^2) + \alpha_k G^2\Bigg(1+\frac{3.42}{1-\lambda}\Bigg)^2.
\end{equation}
The last equation follows from the result of Lemma~\ref{lem_4}. By summing up the last inequality from $k=1$ to $K$, and setting \[\alpha_k=\frac{\sqrt{N}D}{G\Bigg(1+\frac{3.42}{1-\lambda}\Bigg)\sqrt{k}},\] with $\frac{1}{\alpha_0}$ := 0 we obtain
\begin{equation}\label{eq_21_1}
    2\sum_{k=1}^K\nabla V_k(\mathbf{x}_k)^T(\mathbf{x}_k-\mathbf{x}_*)\leq \sum_{k=1}^K\frac{1}{\alpha_k}(\|\mathbf{x}_k-\mathbf{x}_*\|^2-\|\mathbf{x}_{k+1}-\mathbf{x}_*\|^2)+\sum_{k=1}^K\alpha_k G^2\Bigg(1+\frac{3.42}{1-\lambda}\Bigg)^2.
\end{equation}
The right hand side of the Eq.~\ref{eq_21_1} can be bounded above by \[\sum_{k=1}^K\|\mathbf{x}_k-\mathbf{x}_{*}\|^2\Bigg(\frac{1}{\alpha_k}-\frac{1}{\alpha_{k-1}}\Bigg)+G^2\Bigg(1+\frac{3.42}{1-\lambda}\Bigg)^2\sum_{k=1}^K\alpha_k,\]which follows from $\frac{1}{\alpha_0}=0$ and $\|\mathbf{x}_{K+1}-\mathbf{x}_{*}\|^2\geq 0$.
Hence, combining the convexity, we get
\begin{equation}
    2\sum_{k=1}^K(V_k(\mathbf{x}_k)-V_k(\mathbf{x}_*))\leq\sum_{k=1}^K\|\mathbf{x}_k-\mathbf{x}_{*}\|^2\Bigg(\frac{1}{\alpha_k}-\frac{1}{\alpha_{k-1}}\Bigg)+G^2\Bigg(1+\frac{3.42}{1-\lambda}\Bigg)^2\sum_{k=1}^K\alpha_k.
\end{equation}
Using Assumption~\ref{assump_2}, the following is obtained
\begin{equation}
    2\sum_{k=1}^K(V_k(\mathbf{x}_k)-V_k(\mathbf{x}_*))\leq N D^2\sum_{k=1}^K\Bigg(\frac{1}{\alpha_k}-\frac{1}{\alpha_{k-1}}\Bigg)+G^2\Bigg(1+\frac{3.42}{1-\lambda}\Bigg)^2\sum_{k=1}^K\alpha_k.
\end{equation}
Thus, we have,
\begin{equation}
    2\sum_{k=1}^K(V_k(\mathbf{x}_k)-V_k(\mathbf{x}_*))\leq N D^2\frac{1}{\alpha_K}+G^2\Bigg(1+\frac{3.42}{1-\lambda}\Bigg)^2\sum_{k=1}^K\alpha_k\leq 3\sqrt{N}DG\Bigg(1+\frac{3.42}{1-\lambda}\Bigg)\sqrt{K}.
\end{equation}
The last inequality follows from $\sum_{k=1}^K\frac{1}{\sqrt{k}}\leq 2\sqrt{K}$. Hence,
\begin{equation}
    \sum_{k=1}^K(V_k(\mathbf{x}_k)-V_k(\mathbf{x}_*))\leq \frac{3\sqrt{N}DG}{2}\Bigg(1+\frac{3.42}{1-\lambda}\Bigg)\sqrt{K},
\end{equation}
which completes the proof.
\end{proof}

\section{Applications to DSGD}\label{app_f}
\subsection{Distributed ERM}
It has been well known that for machine learning problems the distributed empirical risk minimization can be expressed as follows, 
\begin{equation}\label{dis_erm_1}
    \textnormal{min}_{x\in\mathbb{R}^n}\frac{1}{S}\sum_{i=1}^N\sum_{s\in \mathcal{D}_i}f^i_s(x)
\end{equation}
where $S$ is the cardinality of the dataset $\mathcal{D}$, and $f:\mathbb{R}^n\to\mathbb{R}$, $\mathcal{D}_i$ is data subset owned by the agent $i$ and the cardinality of $\mathcal{D}_i$ is $S_i$ such that $S=\sum_{i=1}^NS_i$. Eq.~\ref{dis_erm_1} can be rewritten as the following constrained minimization formulation~\citep{berahas2018balancing}
\begin{equation}\label{dis_erm_2}
    \textnormal{min}_{x^i\in\mathbb{R}^n}\frac{1}{S}\sum_{i=1}^N\sum_{s\in S_i}f^i_s(x^i),\;\textnormal{s.t.}\;x^i=x^j, j\in Nb(i)
\end{equation}
In a compact form, we can further rewrite Eq.~\ref{dis_erm_2} into
\begin{equation}
    \textnormal{min}_{x^i\in\mathbb{R}^n}\;\mathcal{F}(\mathbf{x})=\sum_{i=1}^Nf^i(x^i),\;\textnormal{s.t.}\;(\Pi\otimes\mathbf{I}_n)\mathbf{x}=\mathbf{x}
\end{equation}
where $f^i(\cdot):=\sum_{s\in S_i}f^i_s(\cdot)$, $\mathbf{x}\in\mathbb{R}^{nN}$ is the concatenation of all local $x^i$'s as shown before and $\Pi$ is mixing matrix of size $\mathbb{R}^{N\times N}$. $\mathbf{I}_n$ is the identity matrix of dimension of $n$ and $\otimes$ is the Kronecker product such that $\Pi\otimes\mathbf{I}_n\in\mathbb{R}^{nN\times nN}$. To simplify the analysis, $n$ is set to 1 and the analysis presented in this paper can be extensively applicable to the case where $n> 1$. Hence, we have
\begin{equation}\label{eq5}
    \textnormal{min}_{x^i\in\mathbb{R}}\;\mathcal{F}(\mathbf{x})=\sum_{i=1}^Nf^i(x^i),\;\textnormal{s.t.}\;\Pi\mathbf{x}=\mathbf{x}    
\end{equation}
We can observe that Eq.~\ref{eq5} has a hard constraint on the $\mathbf{x}$ with the mixing matrix $\Pi$. Relaxing this hard constraint, we perceive the problem in another perspective to get the following formulation 
\begin{equation}\label{uncon}
    \textnormal{min}_{x^i\in\mathbb{R}}\;V(\mathbf{x})=\sum_{i=1}^Nf^i(x^i) + c\mathbf{x}^T(\mathbf{I}_N-\Pi)\mathbf{x}
\end{equation}
where $c$ is a regularity parameter. It has been shown in~\citep{berahas2018balancing} that the offline distributed gradient descent (DGD) method can be thought of as a gradient method on Eq.~\ref{uncon} by applying a unit steplength. We next briefly introduce the offline DGD to investigate how the update law is related to the loss $V(\mathbf{x})$.
In a compact form~\citep{jiang2017collaborative,berahas2018balancing}, the offline DGD is expressed as
\begin{equation}
    \mathbf{x}_{k+1} = \Pi\mathbf{x}_k - \alpha\nabla \mathcal{F}(\mathbf{x}_k)
\end{equation}
where $\alpha> 0$ is a step size and $\nabla \mathcal{F}(\mathbf{x}_k)$ is the gradient of $\mathcal{F}$ at the point $\mathbf{x}_k$. Immediately, it can be rewritten as
\begin{equation}
    \mathbf{x}_{k+1} = \mathbf{x}_k - \mathbf{x}_k + \Pi\mathbf{x}_k - \alpha\nabla \mathcal{F}(\mathbf{x}_k) = \mathbf{x}_k - \alpha(\nabla \mathcal{F}(\mathbf{x}_k) + \alpha^{-1}(\mathbf{I}_N-\Pi)\mathbf{x}_k)
\end{equation}
Therefore, one can observe that the equivalent gradient of $\nabla \mathcal{F}(\mathbf{x}_k) + \alpha^{-1}(\mathbf{I}_N-\Pi)\mathbf{x}_k$ is equal to $\nabla V(\mathbf{x}_k)$ when $c=\frac{1}{2\alpha}$. 
\subsection{Sublinear Convergence Rate for DSGD}
A special case of distributed online convex optimization is the setting of stochastic optimization that has been studied well in literature. We have introduced the distributed ERM that has been adopted widely for the stochastic optimization above. This holds for most problems due to the difficulty of calculating the expectation as follows.
\[\mathbb{E}\Bigg[\sum_{i=1}^Nf^i(x^i;\omega)\Bigg]\approx\frac{1}{S}\sum_{i=1}^N\sum_{s\in\mathcal{D}_i}f^i_s(x^i;\omega),\;\textnormal{s.t.}\;x^i=x^j,\;i,j\in Nb(i)\]where $\omega$ is the random seed to represent the sample-based behavior. It is assumed that for agent $i$, $f^i$ is sampled at an independent and identically distributed (i.i.d.) realization $\omega_k$ of $\omega$ at each iteration. Moving the constraint into the loss and rewriting the last equation results in the analogous form of Eq.~\ref{uncon}. 
For the convenience of the analysis, we define\[\mathbb{E}[H(\mathbf{x};\omega)]=\mathbb{E}\Bigg[\sum_{i=1}^Nf^i(x^i;\omega)+c\mathbf{x}^T(\mathbf{I}_N-\Pi)\mathbf{x}\Bigg].\] Hence, in an online perspective, we have $V_k(\mathbf{x})=H(\mathbf{x};\omega_k)$. Define $\bar{\mathbf{x}}_K=\frac{1}{K}\sum_{k=1}^K\mathbf{x}_k$. We attain 
\begin{equation}
    H(\bar{\mathbf{x}}_K;\omega_k)=V_k(\bar{\mathbf{x}}_K)\leq\frac{1}{K}\sum_{k=1}^KV_k(\mathbf{x}_k)
\end{equation}
The last inequality follows from the Jensen inequality. Thus, we have
\begin{equation}
    \mathbb{E}[H(\bar{\mathbf{x}}_K;\omega_k)]-\textnormal{min}\;\mathbb{E}[H]\leq\frac{1}{K}\sum_{k=1}^KV_k(\mathbf{x}_k)-V_k(\mathbf{x}_*)\leq\sum_{k=1}^K\frac{1}{K}\nabla V_k(\mathbf{x}_k)^T(\mathbf{x}_k-\mathbf{x}_*),
\end{equation}
which follows from the convexity of $V_k(\mathbf{x}_k)$. $\textnormal{min}\;\mathbb{E}[H]$ is the lower bound, which means that $H(\mathbf{x};\omega)$ is bounded below. This assumption has been generically made for stochastic optimization. According to Theorem~\ref{theorem_2}, the convergence rate of DSGD is attained as follows
\begin{equation}
    \mathbb{E}[H(\bar{\mathbf{x}}_K;\omega)]-\textnormal{min}\;\mathbb{E}[H]\leq\frac{\sqrt{N}DC}{\sqrt{K}}=\sqrt{N}DG\Bigg(1+\frac{2}{1-\lambda}\Bigg)\frac{1}{\sqrt{K}}.
\end{equation}
The above convergence rate matches the existing literature~\citep{jiang2017collaborative,lian2017can} for distributed stochastic gradient method. Thus, we have shown the translation from distributed online optimization to distributed stochastic optimization. In the centralized setting, this has been well studied and please refer to~\citep{belmega2018online} for more detail.
\section{Comparisons Between Different Methods}\label{table_section}
We summarize our findings and compare the proposed algorithms with a few existing and recently developed algorithms in Table~\ref{table:findings}. We also provide detailed discussion for them.
Table~\ref{table:findings} shows the regret bounds for different online optimization methods ranging from centralized (i.e., OGD~\citep{hazan2016introduction}, OMD~\citep{hazan2016introduction}, RFTL~\citep{hazan2016introduction}, FTPL~\citep{hazan2016introduction,suggala2019online}, ONGD~\citep{gao2018online}) to distributed settings. For the distributed online methods we omit the number of agents in $\mathcal{O}(\cdot)$, instead focusing more on the time complexity, as the number of time steps is generically much larger than the number of agents, which, however, has been shown to have an impact on the regret in~\citep{zhao2019decentralized} and our work. We can observe that corresponding to various methods, the classes of loss functions, the definitions of step sizes, and the assumptions imposed on the gradients are different. Additionally, the static or dynamic regret is used as the metrics for evaluating the regret bound. Note also that for FTPL and FTPL2, they do not have explicit step sizes in the update rules, but having a parameter that plays a similar role as the step size. We now discuss the comparisons in detail.

The \textbf{OGD} is a widely adopted simple method for the centralized online optimization and has been provably shown to achieve the best regret bound when the loss function is either convex or strongly convex with the only bounded gradient assumption. The logarithmic regret for the strongly convex losses still remains the best regret bound in literature, although the \textbf{DOGT}~\citep{zhang2019distributed}  could enjoy the constant regret bound, when the gradient is Lipschitz continuous. Additionally, the constant step size for the DOGT has to satisfy a more restrictive condition according to~\citep{zhang2019distributed}. The OGD is different from the DOGT in terms of application settings and so far there is  no reported result to show the direct comparison between the OGD and the online gradient tracking (OGT) method. Moreover, if the loss becomes convex, the regret bound is unknown for the DOGT. Our proposed \textbf{OLGD} extends the OGD to a distributed setting and inherently attains the same order of regret bound with the composite regret.

The \textbf{OMD}, \textbf{RFTL} and \textbf{FTPL} share similarities for the convex loss while the FTPL has been recently extended to a more general non-convex scenario, which is still a quite challenging problem for the community. Analogously, the OLGD is able to achieve the same growth order of regret as both the OMD and RFTL obtain. These two approaches essentially leverage the regularization to stabilize the decision making process for the online learning and most decentralized variants such as a few to be introduced next follow the similar regularization techniques. Nonetheless, for the regret, these methods still depend on either the $\textnormal{S-regret}_K$ or $\textnormal{D-regret}_K$, and have not directly taken into account the agent variations in the regret minimization. Also, in the composite regret, when $c=0$, we can observe that it degenerates to the regular $\textnormal{S-regret}_K$ or $\textnormal{D-regret}_K$.

The \textbf{ONGD} was proposed for the pseudo-convex loss that relaxes the convex loss, which, however, is not as difficult as the general non-convex loss. The typical first-order necessary condition that $\nabla\mathcal{F}(\mathbf{x_*})=0$ enables $\mathbf{x_*}$ to be a local minimum instead of a maximum or a saddle point. In our paper, the corresponding \textbf{CONGD} is developed with achieving the same order of regret bound, which includes the temporal change $P_K$. Besides, due to the composite regret and the network topology, the spatial change is also incorporated such that we can quantitatively attain the impact of the spatiotemporal change on the regret bound. We also notice that the order of $P_K=K^\epsilon (\epsilon\in(0,1))$ could significantly affect the bound, although it is always sublinear.

The \textbf{DPD-MD}~\citep{sharma2020distributed}, \textbf{DOPD-PS}~\citep{li2018distributed} and \textbf{DOPD-DMD}~\citep{yi2020distributed} are in the family of primal-dual methods, which needs to cope with more complex constraints instead of simple box constraints. This also results in more sophisticated updates rules and probably the worse regret bounds. For the DOPD-PS, we can see that the regret bound is up to $\mathcal{O}(K^{\frac{9}{10}})$, while for the DPD-MD, when properly selecting $a$ and $b$, the minimum regret bound is still larger than $\mathcal{O}(\sqrt{K})$. While the DOPD-DMD has the closest regret bound as the OLGD has when $\kappa=\frac{1}{2}$, which is also dependent on the path variation $P_K$ due to the non-stationary regret. It has more complicated update rule to rely on the dual variable, which could increase the computational complexity. We realize that a more recent work, \textbf{ODCMD}~\citep{yuan2020distributed}, investigated the loss that is a composite function such that the regret adopted in their paper has a similar form of the composite regret. However, after carefully checking, we have found that their form was not motivated by the variations among agents, but only a regularization that is not necessarily a consensus term. Additionally, they only focused on a static form such that the techniques studied in that paper may not be applicable to dynamic regrets required. To study the property of non-convex losses in the distributed online setting, \textbf{ODM}~\citet{lu2021online} was developed by leveraging the first-order optimality condition-based regret and the authors have shown the proposed algorithm achieved the standard regret $\mathcal{O}(\sqrt{K})$. However, the presented techniques have no guarantee on the applications to the general non-convex losses. \textbf{DINOCO} is the only method in this paper that can handle the regret bound in the distributed online setting with general non-convex losses. We leverage the non-convex FTPL and develop DINOCO that is suitable for solving the online non-convex learning problems in a distributed setting. It provides useful insights on how the network topology and the number of agents affect the regret bound. More detail has been provided in Section~\ref{ftpl_lya}.
\end{document}